\documentclass[10pt,twocolumn,letterpaper]{article}

\usepackage{cvpr}
\usepackage{times}
\usepackage{epsfig}
\usepackage{graphicx}
\usepackage{amsmath}
\usepackage{amssymb}
\usepackage{ulem}
\usepackage{amsthm}
\usepackage{subcaption, booktabs}
\usepackage{tablefootnote}

\usepackage{mathrsfs}
\usepackage{color}
\newtheorem{theorem}{Theorem}

\usepackage[breaklinks=true,bookmarks=false]{hyperref}

\cvprfinalcopy % *** Uncomment this line for the final submission

\def\cvprPaperID{4753} % *** Enter the CVPR Paper ID here

\begin{document}

%%%%%%%%% TITLE
\title{Robustness via curvature regularization, and vice versa}

\author{Seyed-Mohsen Moosavi-Dezfooli\thanks{The first two authors contributed equally to this work.}\;\thanks{\'Ecole Polytechnique F\'ed\'erale de Lausanne}
\\
{\tt\small seyed.moosavi@epfl.ch}\\\\
Jonathan Uesato\footnotemark[3]\\
{\tt\small juesato@google.com}
\and
Alhussein Fawzi\footnotemark[1]\;\thanks{DeepMind}\\
{\tt\small afawzi@google.com}\\\\
Pascal Frossard\footnotemark[2]\\
{\tt\small pascal.frossard@epfl.ch}
}
\maketitle
% \thispagestyle{empty}

%%%%%%%%% ABSTRACT
\begin{abstract}
State-of-the-art classifiers have been shown to be largely vulnerable to adversarial perturbations. One of the most effective strategies to improve robustness is adversarial training. In this paper, we investigate the effect of adversarial training on the geometry of the classification landscape and decision boundaries. We show in particular that  adversarial training leads to a significant decrease in the curvature of the loss surface with respect to inputs, leading to a drastically more  ``linear'' behaviour of the network. Using a locally quadratic approximation, we provide theoretical evidence on the existence of a strong relation between large robustness and small curvature. To further show the importance of reduced curvature for improving the robustness, we propose a new regularizer that directly minimizes curvature of the loss surface, and leads to adversarial robustness that is on par with adversarial training. Besides being a more efficient and principled alternative to adversarial training, the proposed regularizer confirms our claims on the importance of exhibiting quasi-linear behavior in the vicinity of data points in order to achieve robustness.
\end{abstract}

%%%%%%%%% BODY TEXT
\section{Introduction}

Adversarial training has recently been shown to be one of the most successful methods for increasing the robustness to adversarial perturbations of deep neural networks \cite{goodfellow2014, moosavi2015deepfool, madry2017towards}. This approach consists in training the classifier on \textit{perturbed} samples, with the aim of achieving higher robustness than a network trained on the original training set. Despite the importance and popularity of this training mechanism, the effect of adversarial training on the geometric properties of the classifier -- its loss landscape with respect to the input and decision boundaries -- is not well understood. In particular, how do the decision boundaries and loss landscapes of adversarially trained models compare to the ones trained on the original dataset?

In this paper, we analyze such properties and show that one of the main effects of adversarial training is to induce a significant \textit{decrease} in the curvature of the loss function and decision boundaries of the classifier. More than that, we show that such a geometric implication of adversarial training allows us to explain  the high robustness of adversarially trained models. To support this claim,
we follow a \textit{synthesis} approach, where a new regularization strategy, Curvature Regularization (CURE), encouraging small curvature is proposed and shown to achieve robustness levels that are comparable to that of adversarial training. This highlights the importance of small curvature for improved robustness. In more detail, our contributions are summarized as follows:
\begin{itemize}
    \item We empirically show that adversarial training induces a significant \textit{decrease in the curvature} of the decision boundary and loss landscape \textit{in the input space}.
    \vspace{-2mm}
    \item Using a quadratic approximation of the loss function, we establish upper and lower bounds on the robustness to adversarial perturbations with respect to the curvature of the loss.
    These bounds confirm the existence of a relation between low curvature and high robustness.
    \vspace{-2mm}
    \item Inspired by the implications of adversarially trained networks on the curvature of the loss function and our theoretical bounds, we propose an efficient regularizer that encourages small curvatures. On standard datasets (CIFAR-10 and SVHN), we show that the proposed regularizer leads to a significant boost of the robustness of neural networks, comparable to that of adversarial training.
\end{itemize}
The latter step shows that the proposed regularizer can be seen as a more efficient alternative to adversarial training. More importantly, it shows that the effect of adversarial training on the curvature reduction is not a mere by-product, but rather a driving effect that causes the robustness to increase. We stress here that the main focus of this paper is mainly on the latter -- analyzing the geometry of adversarial training -- rather than outperforming adversarial training.

\paragraph{Related works.} The large vulnerability of classifiers to adversarial perturbations has first been highlighted in \cite{biggio2013evasion, szegedy2013intriguing}. Many algorithms aiming to improve the robustness have since then been proposed \cite{goodfellow2014, shaham2015understanding, madry2017towards, cisse2017parseval, alemi2016deep}. In parallel, there has been a large body of work on designing improved attacks \cite{moosavi2015deepfool, madry2017towards}, which have highlighted that many of the proposed defenses obscure the model rather than make the model truly robust against all attacks \cite{uesato2018adversarial, athalye2018obfuscated}. One defense however stands out -- adversarial training -- which has shown to be empirically robust against all designed attacks. The goal of this paper is to provide an analysis of this phenomenon, and propose a regularization strategy (CURE), which mimics the effect of adversarial training. On the analysis front, many works have analyzed the existence of adversarial examples, and proposed several hypotheses for their existence \cite{fawzi2015analysis, gilmer2018adversarial, tanay2016boundary, fawzi2018nips, hein2017formal}. In \cite{goodfellow2014}, it is hypothesized that networks are not robust as they exhibit a ``too linear'' behavior. We show here that linearity of the loss function with respect to the inputs (that is, small curvature) is, on the contrary, beneficial for robustness: adversarial training does lead to much more linear loss functions in the vicinity of data points, and we verify that this linearity is indeed the source of increased robustness. We finally note that prior works have attempted to improve the robustness using gradient regularization \cite{gu2014towards, lyu2015unified, ross2017improving}. However, such methods have not been shown to yield significant robustness on complex datasets, or have not been subject to extensive robustness evaluation. Instead, our main focus here is to study the effect of the \textit{second-order} properties of the loss landscape, and show the existence of a strong connection with robustness to adversarial examples.

\section{Geometric analysis of adversarial training}
\label{sec:geometric_analysis_adv_training}
We start our analysis by inspecting the effect of adversarial training on the geometric properties of the decision boundaries of classifiers. To do so, we first compare qualitatively the decision boundaries of classifiers \textit{with} and \textit{without} adversarial training. Specifically, we examine the effect of \textit{adversarial fine-tuning}, which consists in fine-tuning a trained network with a few extra epochs on adversarial examples.\footnote{While adversarial fine-tuning is distinct from vanilla adversarial training, which consists in training on adversarial images \textit{from scratch}, we use an adversarially fine-tuned network in this paper as it allows to single out the effect of training on adversarial examples, as opposed to other uncontrolled phenomenon happening in the course of vanilla adversarial training.} We consider the CIFAR-10 \cite{krizhevsky2009learning} and SVHN \cite{netzer2011reading} datasets, and use a ResNet-18 \cite{he2015deep} architecture. For fine-tuning on adversarial examples,
we use DeepFool \cite{moosavi2015deepfool}.

Fig.~\ref{fig:normal_sections} illustrates normal cross-sections of the decision boundaries  before and after adversarial fine-tuning for classifiers trained on CIFAR-10 and SVHN datasets.
Specifically, the classification regions are shown in the plane spanned by $(r, v)$, where $r$ is the normal to the decision boundary and $v$ corresponds to a random direction.
\begin{figure}
    \centering
    \begin{subfigure}[b]{0.4\columnwidth}
        \includegraphics[width=0.8\linewidth, page=1]{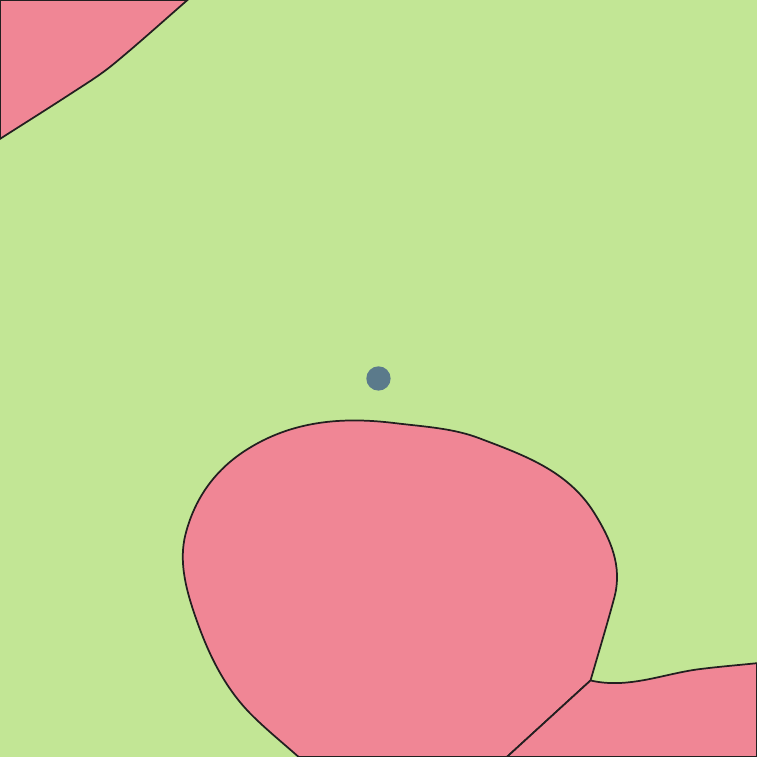}
        \caption{Original (CIFAR-10)}
    \end{subfigure}~
    \begin{subfigure}[b]{0.4\columnwidth}
        \includegraphics[width=0.8\linewidth, page=2]{cross_sections_CIFAR_SVHN.pdf}
        \caption{Finetuned (CIFAR-10)}
    \end{subfigure}
    \begin{subfigure}[b]{0.4\columnwidth}
        \includegraphics[width=0.8\linewidth, page=3]{cross_sections_CIFAR_SVHN.pdf}
        \caption{Original (SVHN)}
    \end{subfigure}~
    \begin{subfigure}[b]{0.4\columnwidth}
        \includegraphics[width=0.8\linewidth, page=4]{cross_sections_CIFAR_SVHN.pdf}
        \caption{Fine-tuned (SVHN)}
    \end{subfigure}
   \caption{\label{fig:normal_sections}Random normal cross-sections of the decision boundary for ResNet-18 classifiers trained on CIFAR-10 (first row) and SVHN (second row). The first column is for classifiers trained on the original dataset, and the second column shows the boundaries after adversarial fine-tuning on 20 epochs for CIFAR-10 and 10 epochs for SVHN. The green and red regions represent the correct class and incorrect classes, respectively. The point at the center shows the datapoint, while the lines represent the different decision boundaries (note that the red regions can include different incorrect classes).}
\end{figure}
In addition to inducing a larger distance between the data point and the decision boundary (hence resulting in a higher robustness), observe that the decision regions of fine-tuned networks are flatter and more regular. In particular, note that the curvature of the decision boundaries decreased after fine-tuning.

\begin{figure*}[ht]
    \begin{subfigure}[b]{0.5\textwidth}
        \centering
        \includegraphics[width=0.6\linewidth]{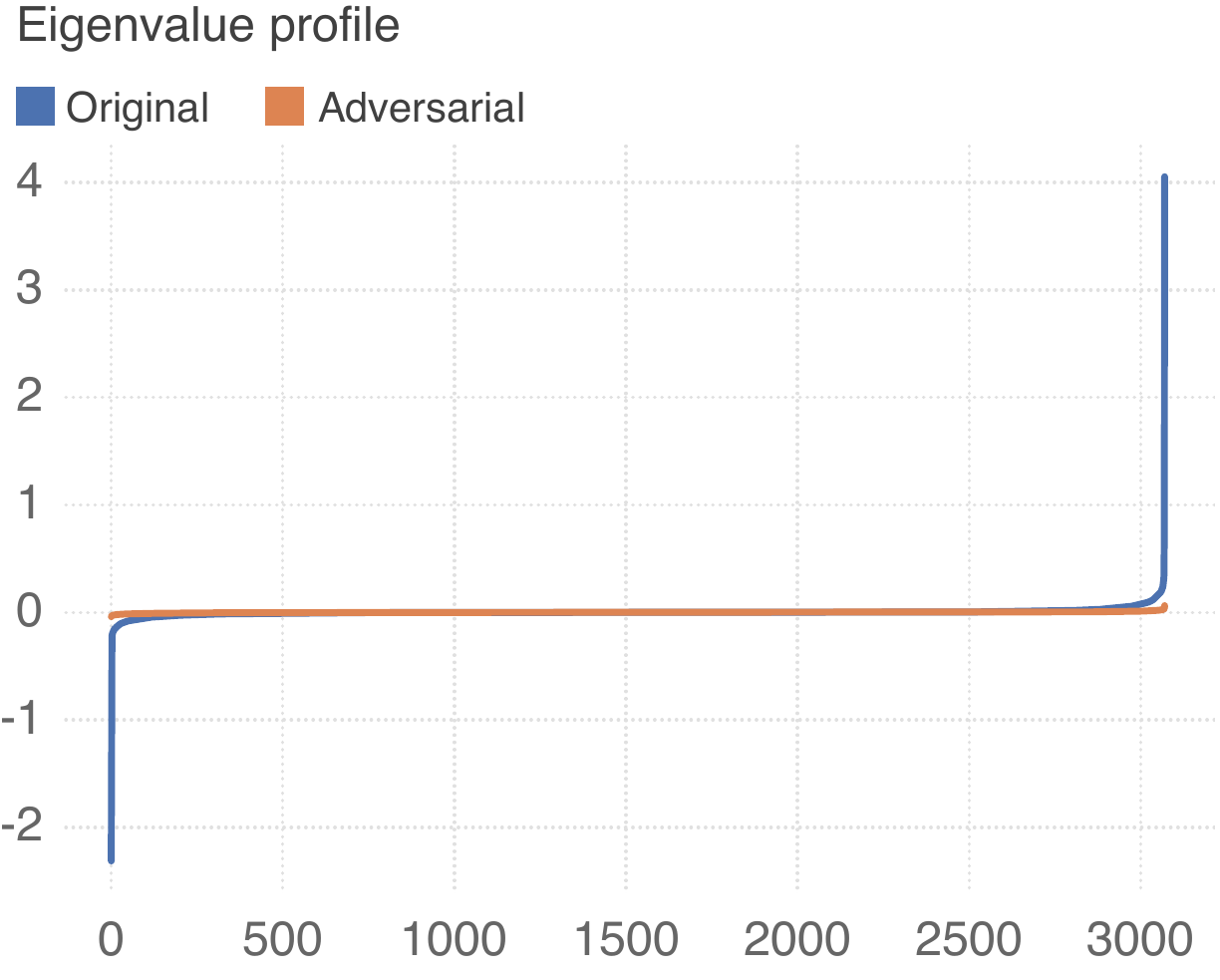}
        \caption{CIFAR-10}
    \end{subfigure}
    \begin{subfigure}[b]{0.5\textwidth}
        \centering
        \includegraphics[width=0.6\linewidth]{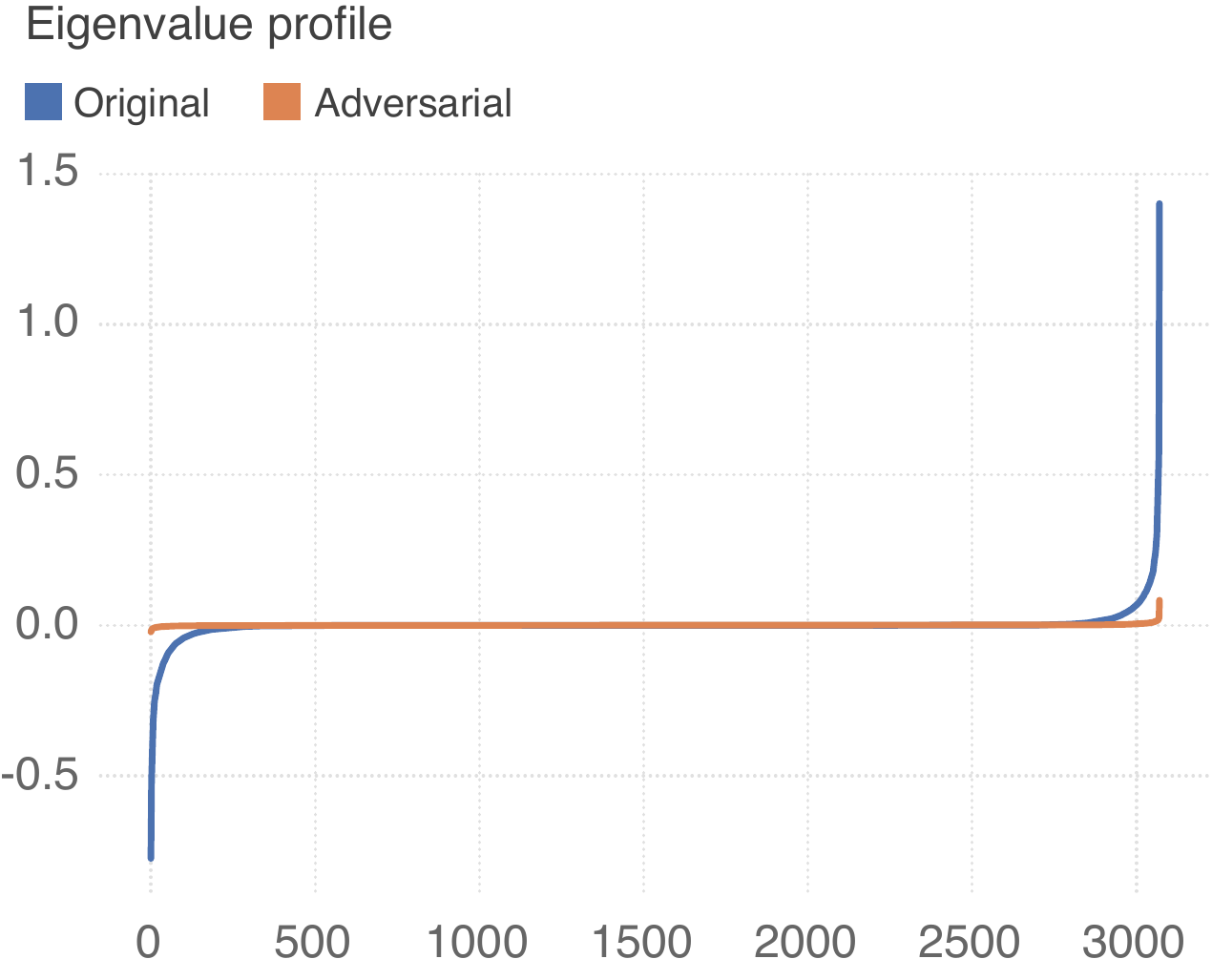}
        \caption{SVHN}
    \end{subfigure}
    \caption{Curvature profiles, which correspond to sorted eigenvalues of the Hessian, of the original and the adversarially fine-tuned networks. Note that the number of eigenvalues is equal to $32 \times 32 \times 3 = 3072$, which corresponds to the number of input dimensions. The ResNet-18 architecture is used.}
    \label{fig:curvature_profiles}
\end{figure*}

To further quantify this phenomenon, we now compute the \textit{curvature profile} of the loss function (with respect to the inputs) before and after adversarial fine-tuning. Formally, let $\ell$ denote the function that represents the loss of the network with respect to the inputs; e.g., in the case of cross-entropy, $\ell(x) = \text{XEnt} (f_{\theta} (x), y)$, where $y$ is the true label of image $x\in\mathbb{R}^d$, and $f_{\theta} (x)$ denotes the logits.\footnote{We omit the label $y$ from $\ell$ for simplicity, as the label can be understood from the context.}
% i.e., $\ell$ is a function from $\mathbb{R}^d \rightarrow \mathbb{R}$
The curvature profile corresponds to the set of eigenvalues of the Hessian matrix $$H = \left( \frac{\partial^2 \ell}{\partial x_i  \partial x_j} \right) \in \mathbb{R}^{d \times d}$$ where $x_i, i = 1, \dots, d$ denote the input pixels. We stress on the fact that the above Hessian is with respect to the inputs, and not the weights of the network. To compute these eigenvalues in practice, we note that Hessian vector products are given by the following for any $z$;
\begin{align}
H z =\frac{\nabla \ell(x+hz) - \nabla \ell(x)}{h} \text{ for } h \rightarrow 0.
\label{eq:finite_diff_analysis}
\end{align}
We then proceed to a finite difference approximation by choosing a finite $h$ in Eq.~(\ref{eq:finite_diff_analysis}). Besides being more efficient than generating the full Hessian matrix (which would be prohibitive for high-dimensional datasets), the finite difference approach has the benefit of measuring  \textit{larger-scale} variations of the gradient (where the scale is set using the parameter $h$) in the neighborhood of the datapoint, rather than an infinitesimal point-wise curvature.
This is crucial in the setting of adversarial classification, where we analyze the loss function in a small neighbourhood of data points, rather than the asymptotic regime $h \rightarrow 0$ which might capture very local (and not relevant) variations of the function.\footnote{For example, using ReLU non-linearities result in a piecewise linear neural network as a function of the inputs. This implies that the Hessian computed at the logits is exactly 0. This result is however very local; using the finite-difference approximation, we focus on larger-scale neighbourhoods.}

Intuitively, small eigenvalues (in absolute value) of $H$ indicate a small curvature of the graph of $\ell$ around $x$, hence implying that the classifier has a ``locally linear'' behaviour in the vicinity of $x$. In contrast, large eigenvalues (in absolute value) imply a high curvature of the loss function in the neighbourhood of image $x$. For example, in the case where the eigenvalues are exactly zero, the function becomes locally linear, hence leading to a flat decision surface.

We compute the curvature profile at $100$ random test samples, and show the average curvature in Fig.~\ref{fig:curvature_profiles} for CIFAR-10 and SVHN datasets.
Note that adversarial fine-tuning has led to a strong decrease in the curvature of the loss in the neighborhood of data points.
To further illustrate qualitatively this significant decrease in curvature due to adversarial training, Fig.~\ref{fig:loss_surfaces} shows the loss surface before and after adversarial training along normal and random directions $r$ and $v$. Observe that while the original network has large curvature in certain directions, the effect of adversarial training is to ``regularize'' the surface, resulting in a smoother, lower curvature (i.e., linear-like) loss.

\begin{figure*}
    \centering
    \begin{subfigure}[t]{0.2\textwidth}
        \includegraphics[width=\linewidth]{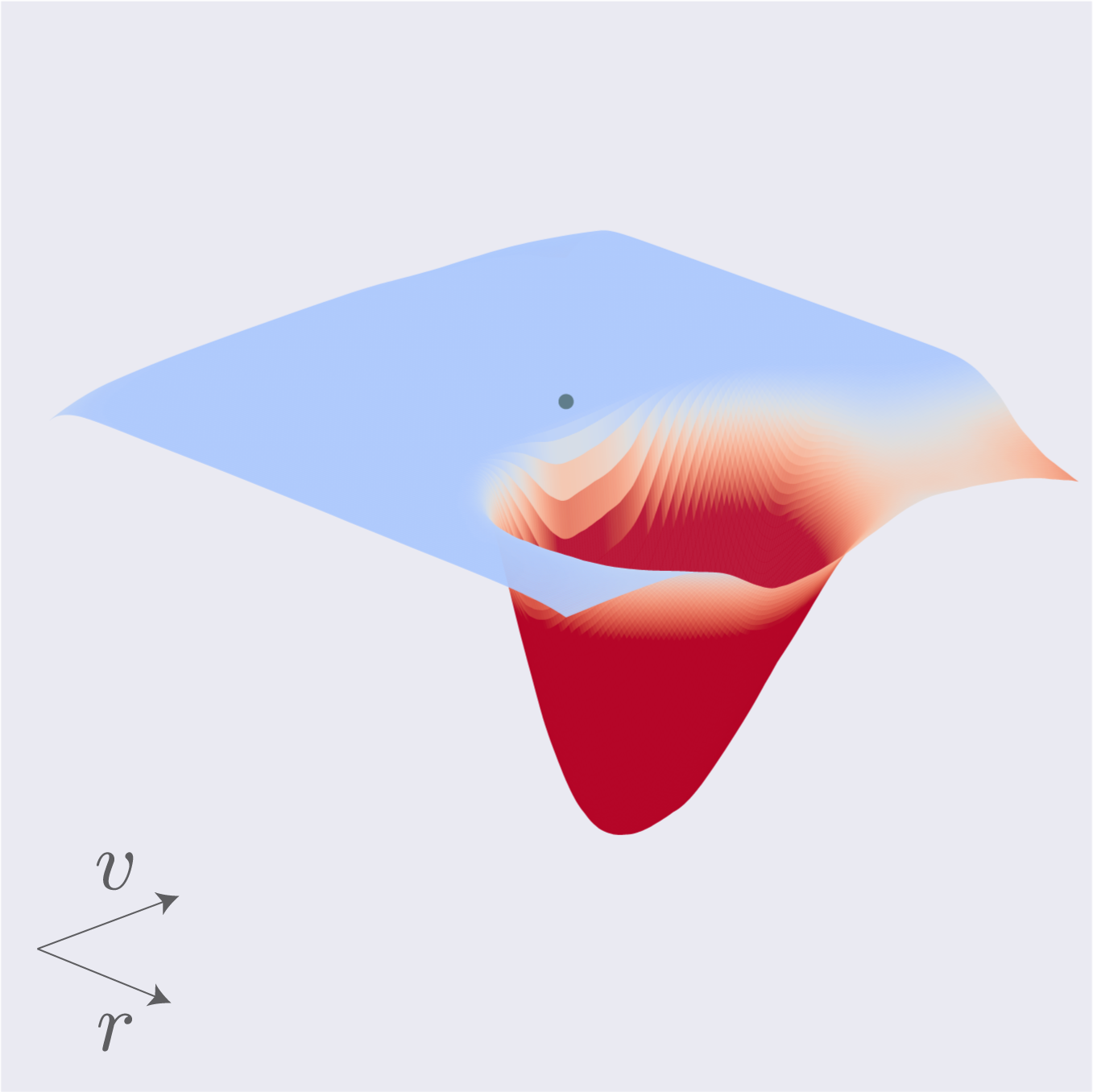}
        \caption{Original (CIFAR-10)}
    \end{subfigure}~
    \begin{subfigure}[t]{0.2\textwidth}
        \includegraphics[width=\linewidth]{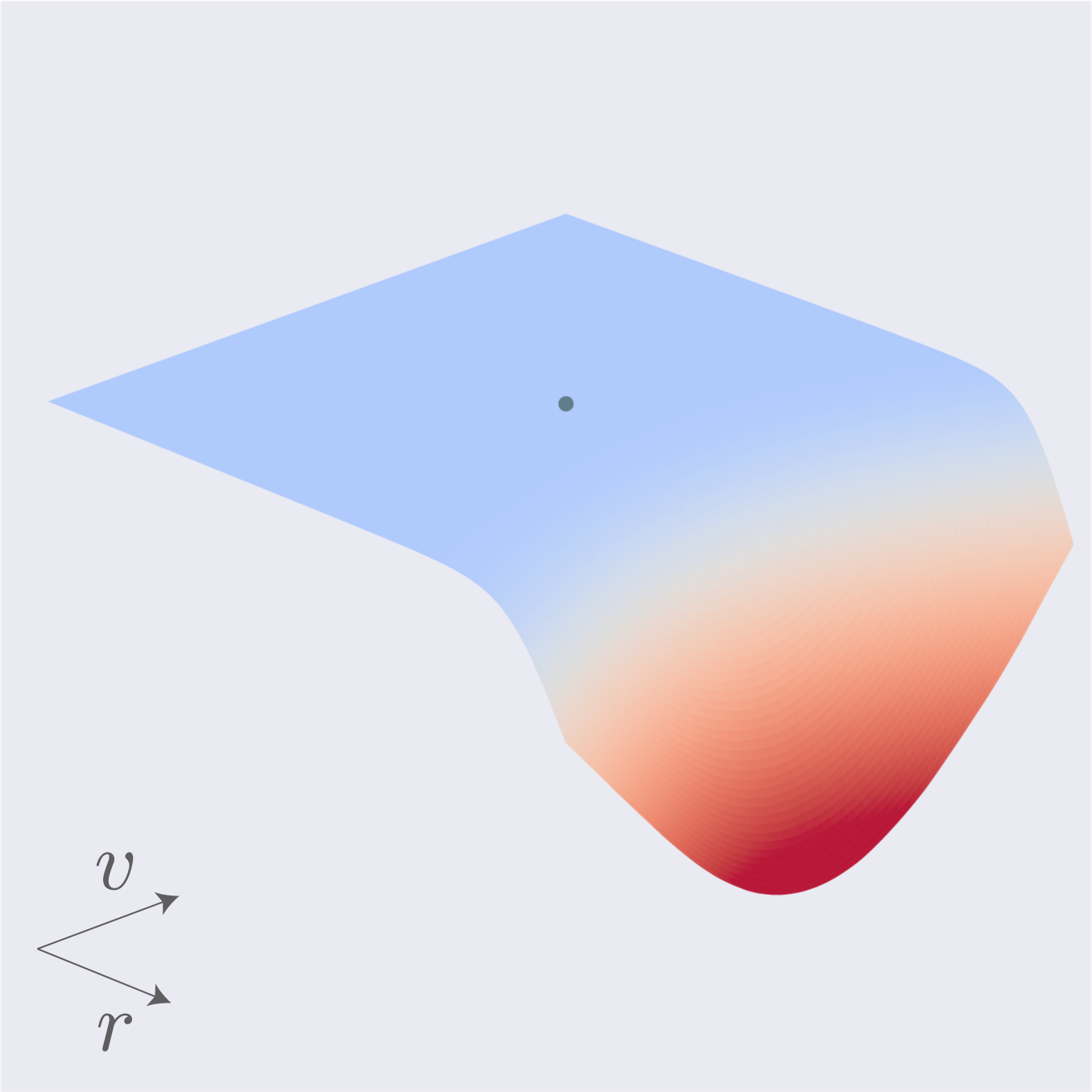}
        \caption{Fine-tuned (CIFAR-10)}
    \end{subfigure}
    ~~
    \begin{subfigure}[t]{0.2\textwidth}
        \includegraphics[width=\linewidth]{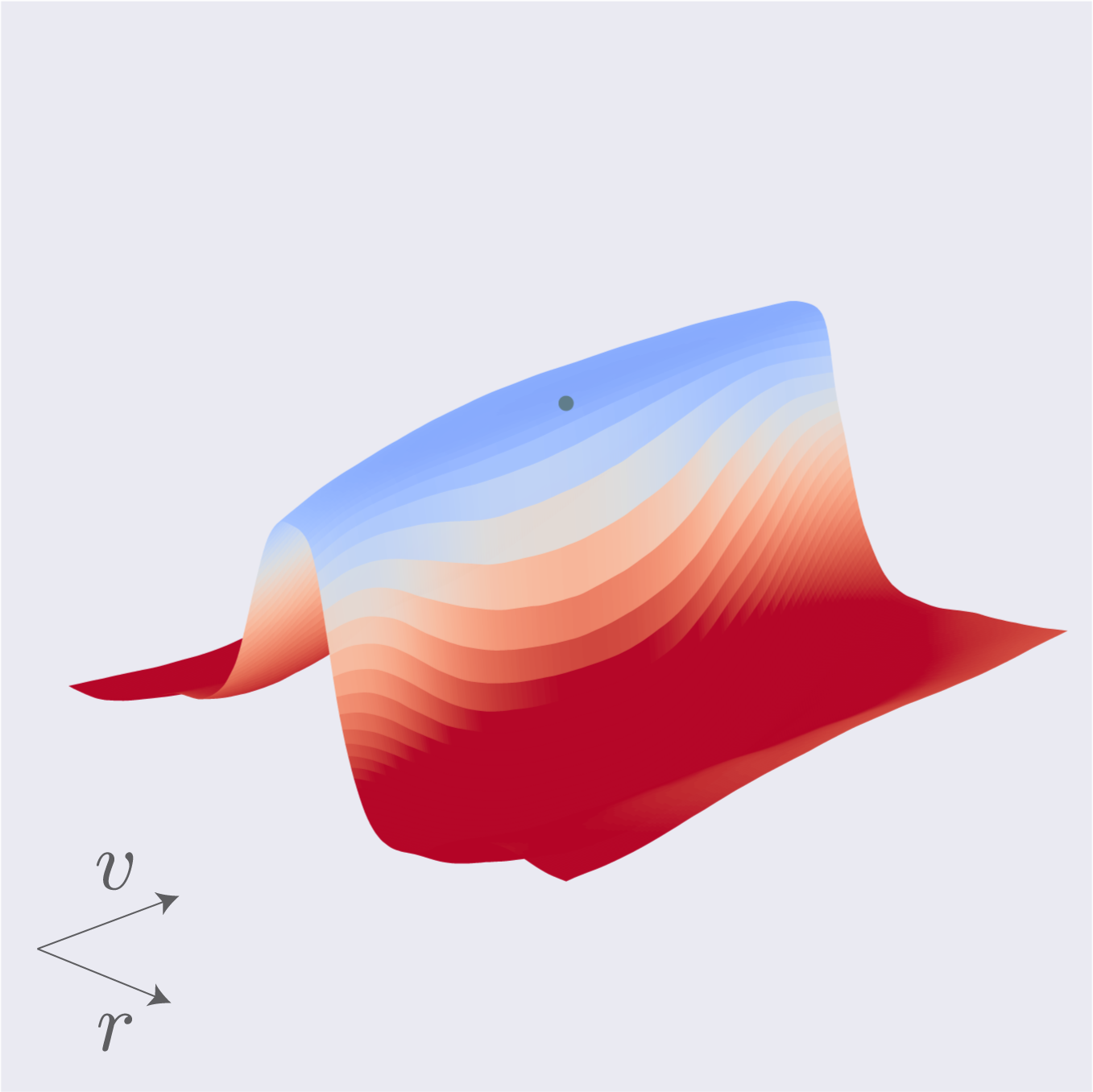}
        \caption{Original (SVHN)}
    \end{subfigure}~
    \begin{subfigure}[t]{0.2\textwidth}
        \includegraphics[width=\linewidth]{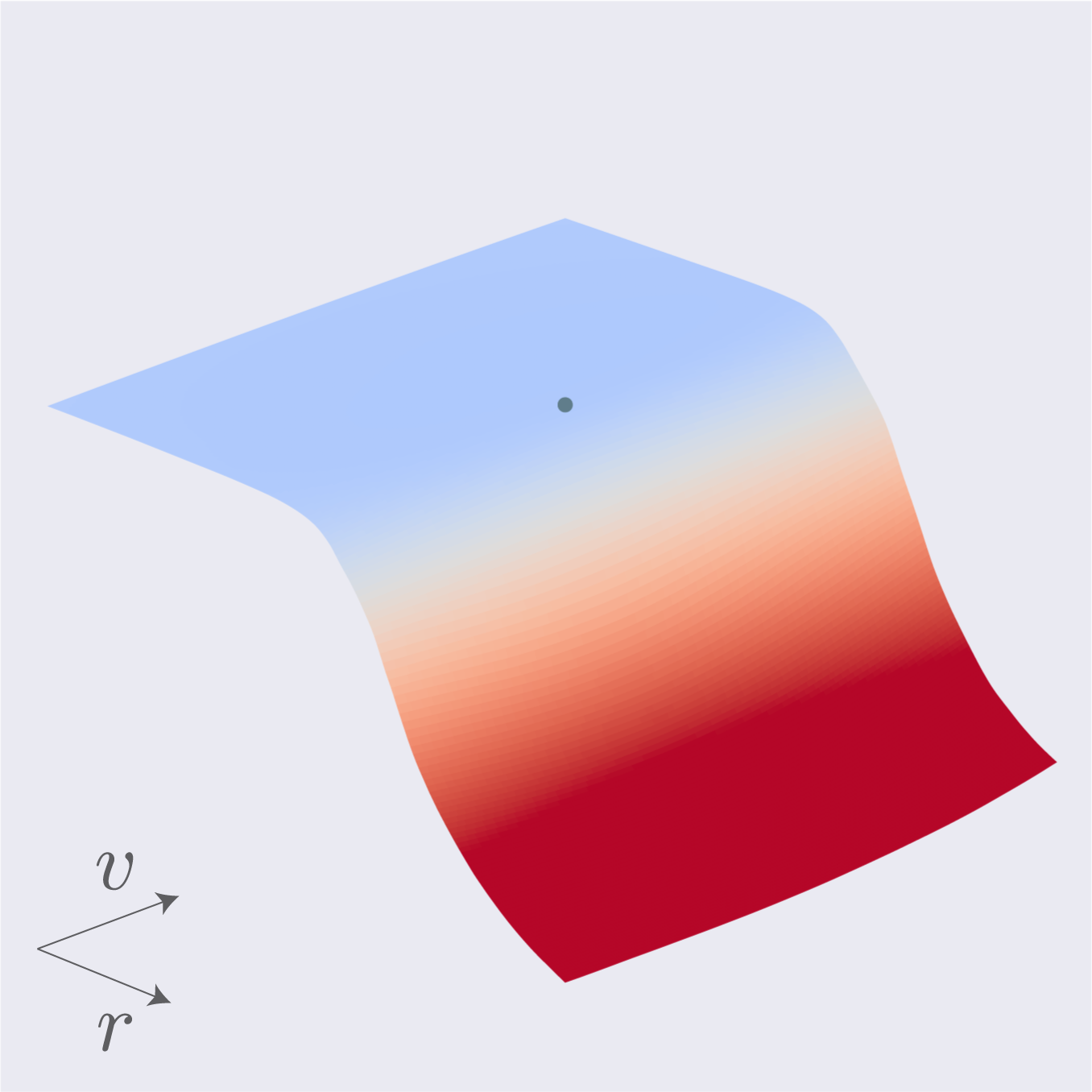}
        \caption{Fine-tuned (SVHN)}
    \end{subfigure}
    \caption{Illustration of the negative of the loss function; i.e., $-\ell(s)$ for points $s$ belonging to a plane spanned by a normal direction $r$ to the decision boundary, and random direction $v$. The original sample is illustrated with a blue dot. The light blue part of the surface corresponds to low loss (i.e., corresponding to the classification region of the sample), and the red part corresponds to the high loss (i.e., adversarial region).}
    \label{fig:loss_surfaces}
\end{figure*}

We finally note that this effect of adversarial training on the loss surface has the following somewhat paradoxical implication: while adversarially trained models are \textit{more robust} to adversarial perturbations (compared to original networks), they are also \textit{easier to fool}, in the sense that simple attacks are as effective as complex ones.
This is in stark contrast with original networks, where complex networks involving many gradient steps (e.g., PGD(20)) are much more effective than simple methods (e.g., FGSM). See Table~\ref{tab:fooling_rate_different_attacks}.
The comparatively small gap between the adversarial accuracies for different attacks on adversarially trained models is a direct consequence of the significant decrease of the curvature of the loss, thereby requiring a small number of gradient steps to find adversarial perturbations.

\begin{table}[]
    \centering
    \begin{tabular}{lcccc}
        \toprule
         & FGSM & $\ell_\infty$-DF &PGD(7) & PGD(20)\\
         \midrule
        Original & $38.0\%$ & $11.0\%$ &$0.5\%$ & $0.2\%$\\
        Fine-tuned & $61.0\%$ &  $57.5\%$ & $57.2\%$ & $56.9\%$\\
    \bottomrule
    \end{tabular}
    \caption{Adversarial accuracies for original and fine-tuned network on CIFAR-10, where adversarial examples are computed with different attacks; FGSM \cite{goodfellow2014}, DF \cite{moosavi2015deepfool} and PGD \cite{madry2017towards}. Perturbations are constrained to have $\ell_{\infty}$ norm smaller than  $\epsilon=4$ (images have pixel values in $[0, 255]$).}
    \label{tab:fooling_rate_different_attacks}
\end{table}

\section{Analysis of the influence of curvature on robustness}

While our results show that adversarial training leads to a decrease in the curvature of the loss, the relation between adversarial robustness and curvature of the loss  remains unclear.
To elucidate this relation, we
 consider a simple binary classification setting between class $1$ and $-1$. Recall that  $\ell(\cdot, 1)$ denotes the function that represents the loss of the network with respect to an input from class $1$. For example, in the setting where the log-loss is considered, we have $\ell(x, 1) = -\log(p(x))$, where $p(x)$ denotes the output of softmax corresponding to class $1$. In that setting, $x$ is classified as class $1$ iff $\ell(x, 1) \leq \log(2)$. For simplicity, we assume in our analysis that $x$ belongs to class $1$ without loss of generality, and hence omit the second argument in $\ell$ in the rest of this section.
We assume that the function $\ell$ can be locally well approximated using a quadratic function; that is, for ``sufficiently small'' $r$, we can write:
\[
\ell(x+r) \approx \ell(x) + \nabla \ell(x)^T r + \frac{1}{2} r^T H r,
\]
where $\nabla \ell(x)$ and $H$ denote respectively the gradient and Hessian of $\ell$ at $x$. Let $x$ be a point classified as class $1$; i.e., $\ell(x) \leq t$, where $t$ denotes the loss threshold (e.g., $t = \log(2)$ for the log loss). For this datapoint $x$, we then define $r^*$ to be the minimal perturbation in the $\ell_2$ sense\footnote{We use the $\ell_2$ norm for simplicity. Using the equivalence of norms in finite dimensional spaces, our result allows us to also bound the magnitude of $\ell_{\infty}$ adversarial perturbations.}, which fools the classifier assuming the quadratic approximation holds; that is,
\begin{align*}
r^* & := \arg\min_{r} \| r \| \text{ s.t. } \ell(x) + \nabla \ell(x)^T r + \frac{1}{2} r^T H r \geq t.
\end{align*}
In the following result, we provide upper and lower bounds on the magnitude of $r^*$ with respect to properties of the loss function at $x$.

\begin{figure}[t]
\centering
\includegraphics[width=0.7\columnwidth]{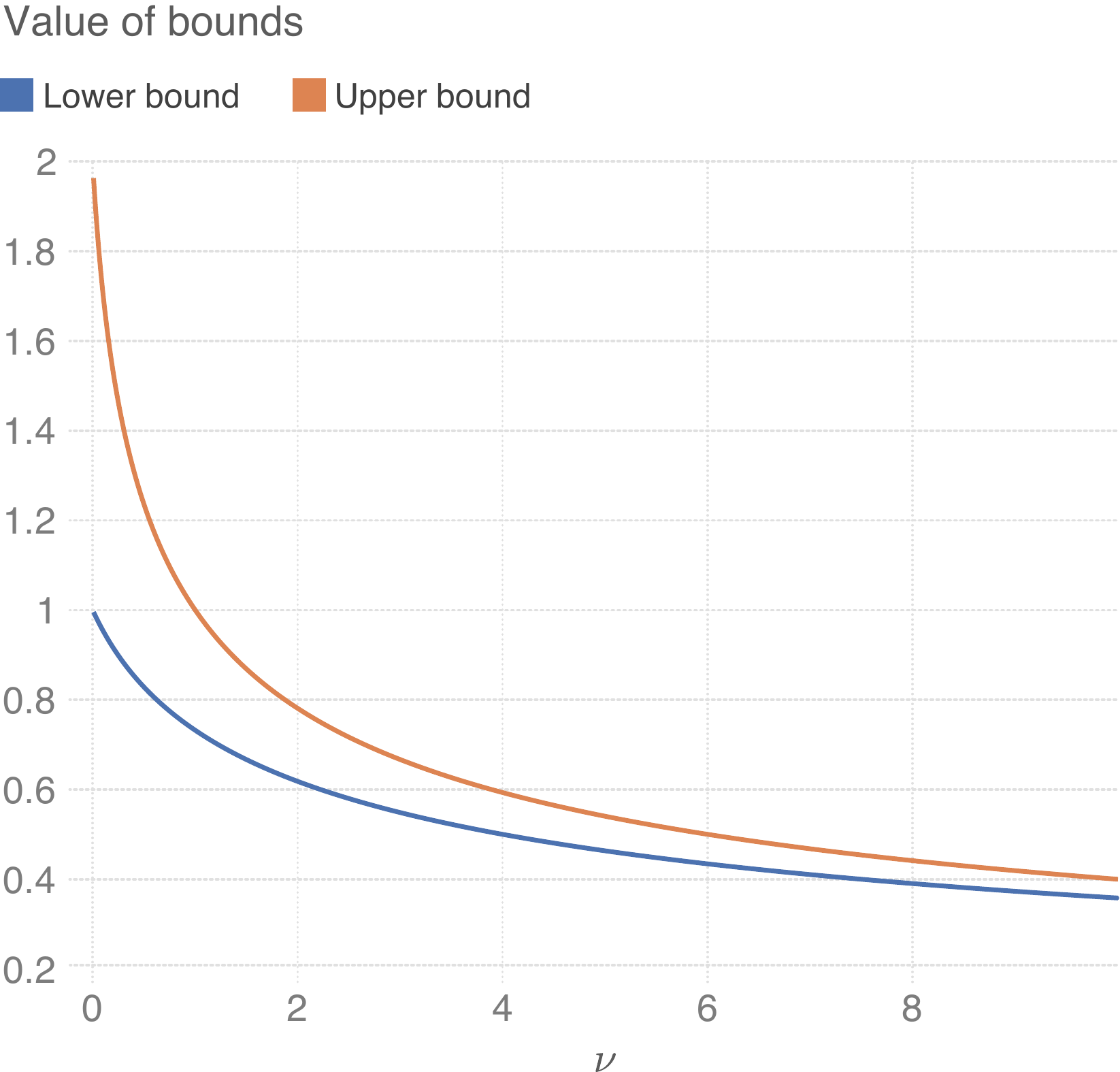}
\caption{\label{fig:dependence_nu_robustness}Illustration of upper and lower bounds in Eq.~(\ref{eq:lower_bound}) and (\ref{eq:upper_bound}) on the robustness with respect to curvature $\nu$. We have set $\| \nabla \ell(x) \| = 1, c = 1, \nabla \ell(x)^T v = 0.5$ in this example.}
\end{figure}

\label{sec:theory}
\begin{figure}
    \centering
    \includegraphics[width=0.8\columnwidth]{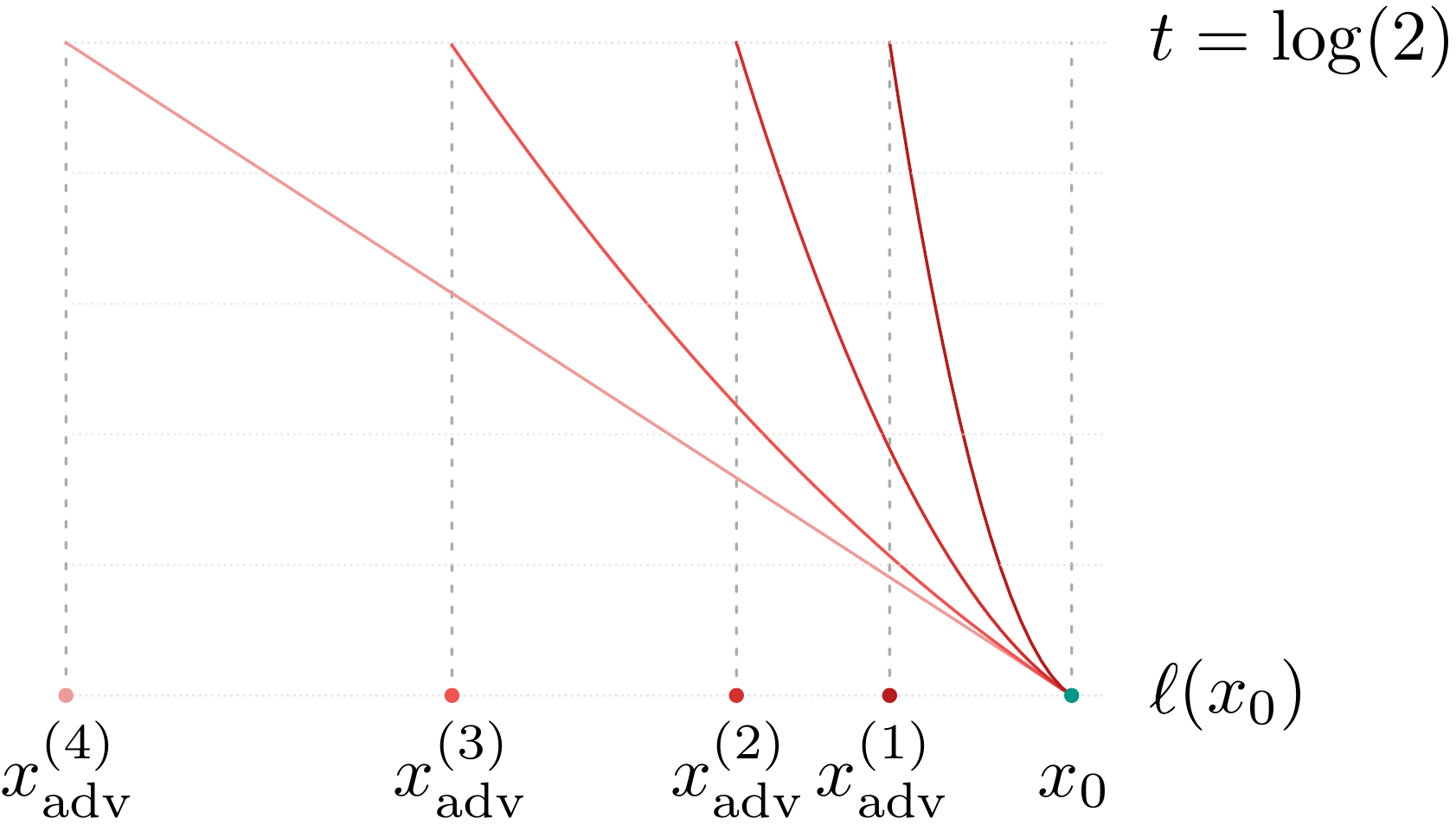}
    \caption{Geometric illustration in 1d of the effect of curvature on the adversarial robustness. Different loss functions (with varying curvatures) are illustrated at the vicinity of data point $x_0$, and $x_{\text{adv}}^{(i)}$ indicate the points at which such losses exceed $t$ (where $t$ is the misclassification threshold). All curves have the same loss and gradient at $x_0$. Note that increasing curvature leads to smaller adversarial examples (i.e., smaller $|x_0 - x_{\text{adv}}^{(i)}|$).}
    \label{fig:illust_curvature}
\end{figure}

\begin{theorem}
\label{thm:theorem_robustness}
Let $x$ be such that $c := t - \ell(x) \geq 0$, and let $g = \nabla \ell(x)$. Assume that $\nu := \lambda_{\max}(H) \geq 0$, and let $u$ be the eigenvector corresponding to $\nu$. Then, we have
\begin{align}
\frac{\| g \|}{\nu} \left( \sqrt{1 + \frac{2 \nu c}{ \| g \|^2}} - 1 \right) &\leq \| r^* \| \label{eq:lower_bound} \\
&\leq \frac{|g^T u|}{\nu} \left( \sqrt{1 + \frac{2 \nu  c}{(g^T u)^2}} - 1\right)
\label{eq:upper_bound}
\end{align}
The above bounds can further be simplified to:
\begin{align*}
\frac{c}{\| g \|} - 2 \nu \frac{c^2}{\| g \|^3} \leq \| r^* \| \leq \frac{c}{|g^T u|}
\end{align*}
\end{theorem}

\begin{proof}
\textbf{Lower bound.} Let $\alpha := \| r^* \|$. We note that $\alpha$ satisfies
\[
- c + \| g \| \alpha + \frac{\nu}{2} \alpha^2 \geq - c + g^T r^* + \frac{1}{2} (r^*)^T H r^* \geq 0.
\]
Solving the above second-order inequality, we get  $\alpha \geq \frac{\| g \|}{\nu} \left( \sqrt{1 + \frac{2 \nu c}{\| g \|^2}} - 1\right)$
or
$
\alpha \leq -\frac{\| g \|}{\nu} \left( \sqrt{1 + \frac{2 \nu c}{\| g \|^2}} + 1 \right).
$
However, since $\alpha \geq 0$, the first inequality holds, which precisely corresponds to the lower bound.

\noindent \textbf{Upper bound.} Let $\alpha \geq 0$. Define $r := \alpha u$, and let us find the minimal $|\alpha|$ such that
\[
- c + g^T r + \frac{1}{2} r^T H r = - c + \alpha g^T u + \frac{\alpha^2 \nu}{2} \geq 0.
\]
We note that the above inequality holds for any $|\alpha| \geq |\alpha_\text{min}|$, with
$ |\alpha_{\text{min}}| = \frac{|g^T u|}{\nu} \left( \sqrt{1 + \frac{2 \nu c}{(g^T u)^2}} - 1\right)$.
Hence, we have that $\| r^* \| \leq |\alpha_{\min}|$, which concludes the proof of the upper bound. The simplified bounds are proven using the inequality $1+\frac{x}{2}-\frac{x^2}{2} \leq \sqrt{1+x} \leq 1+\frac{x}{2}$.
\end{proof}

\textbf{Remark 1. \textit{Increasing} robustness with \textit{decreasing} curvature.} Note that upper and lower bounds on the robustness in Eq.~(\ref{eq:lower_bound}),~(\ref{eq:upper_bound})  \textit{decrease} with increasing curvature $\nu$. To see this, Fig. \ref{fig:dependence_nu_robustness} illustrates the dependence of the bounds on the curvature $\nu$. In other words, under the second order approximation, this shows that \textit{small curvature} (i.e., small eigenvalues of the Hessian) is beneficial to obtain classifiers with higher robustness (when the other parameters are kept fixed). This is in line with our observations from Section \ref{sec:geometric_analysis_adv_training}, where robust models are observed to have a smaller curvature than networks trained on original data. Fig.~\ref{fig:illust_curvature} provides intuition to the decreasing robustness with increasing curvature in a one-dimensional example.

\textbf{Remark 2. Dependence on the gradient.} In addition to the dependence on the curvature $\nu$, note that the upper and lower bounds depend on the gradient $\nabla \ell(x)$. In particular, these bounds  \textit{decrease} with the norm $\| \nabla \ell(x) \|$ (for a fixed direction).
Hence, under the second order approximation, this suggests that the robustness decreases with larger gradients.
However, as previously noted in \cite{uesato2018adversarial, athalye2018obfuscated}, imposing small gradients might provide a false sense of robustness.
That is, while having small gradients can make it hard for gradient-based methods to attack the network, the network can still be intrinsically vulnerable to small perturbations.

\textbf{Remark 3. Bound tightness.} Note that the upper and lower bounds match (and hence bounds are exact) when the gradient $\nabla \ell(x)$ is collinear to the largest eigenvector $u$. Interestingly, this condition seems to be approximately satisfied in practice, as the average normalized inner product $\frac{| \nabla \ell(x)^T u |}{\| \nabla \ell(x) \|_2}$ for CIFAR-10 is equal to $0.43$ before adversarial fine-tuning, and $0.90$ after fine-tuning (average over $1000$ test points). This inner product is significantly larger than the inner product between two typical vectors uniformly sampled from the sphere, which is approximately $\frac{1}{\sqrt{d}} \approx 0.02$. Hence, the gradient aligns well with the direction of largest curvature of the loss function in practice, which leads to approximately tight bounds.

\section{Improving robustness through curvature regularization}

\begin{table*}[ht]
    \centering
    \caption{Adversarial and clean accuracy for CIFAR-10 for original, regularized and adversarially trained models. Performance is reported for ResNet and WideResNet models, and the perturbations are computed using PGD(20). Perturbations are constrained to have $\ell_{\infty}$ norm less than $\epsilon = 8$ (where pixel values are in $[0, 255]$).}
    \begin{tabular}{lccccc}
        \toprule
        & \multicolumn{2}{c}{ResNet-18} && \multicolumn{2}{c}{WideResNet-28$\times$10}\\
        \cmidrule{2-3} \cmidrule{5-6}
         & Clean & Adversarial&& Clean& Adversarial\\
        \midrule
        Normal training & $94.9\%$ & $0.0\%$&& $94.6\%$ & $0.0\%$\\
        CURE & $81.2\%$ & $36.3\%$&& $83.1\%$ & $41.4\%$\\
        Adversarial training~\cite{madry2017towards} & $79.4\%$ & $43.7\%$ && $87.3\%$ & $45.8\%$\\
        \bottomrule
    \end{tabular}
    \label{tab:cifar_regularization_result}
\end{table*}

While adversarial training leads to a regularity of the loss in the vicinity of data points, it remains unclear whether this regularity is the \textit{main} effect of adversarial training, which confers robustness to the network, or it is rather a \textit{byproduct} of a more sophisticated phenomenon. To answer this question, we follow here a \textit{synthesis} approach, where we derive a regularizer which mimics the effect of adversarial training on the loss function -- encouraging small curvatures.
\paragraph{Curvature regularization (CURE) method.} Recall that $H$ denotes the Hessian of the loss $\ell$ at datapoint $x$. We denote by $\lambda_1, \dots, \lambda_d$ the eigenvalues of $H$. Our aim is to penalize large eigenvalues of $H$; we therefore consider a regularizer
$L_r = \sum_i p(\lambda_i),$
where $p$ is a non-negative function, which we set to be $p(t) = t^2$ to encourage all eigenvalues to be small. For this choice of $p$, $L_r$ corresponds to the Frobenius norm of the matrix $H$. We further note that
\[
L_r = \sum_i p(\lambda_i) = \text{trace}(p(H)) = \mathbb{E} (z^T p(H) z) = \mathbb{E} \| H z \|^2,
\]
where the expectation is taken over $z \sim \mathcal{N}(0, I_d)$. By using a finite difference approximation of the Hessian, we have
$Hz \approx \frac{\nabla \ell(x+hz) - \nabla \ell(x)}{h}$,
where $h$ denotes the discretization step, and controls the scale on which we require the variation of the gradients to be small.
Hence, $L_r$ becomes
\begin{align*}
L_r & = \frac{1}{h^2} \mathbb{E} \left\| \nabla \ell(x+hz) - \nabla \ell(x) \right\|^2.
\end{align*}
The above regularizer involves computing an expectation over $z \sim \mathcal{N}(0, I_d)$, and penalizes large curvatures along all directions equally. Rather than approximating the above with an empirical expectation of $\| H z \|^2$ over isotropic directions drawn from $\mathcal{N}(0, I_d)$, we instead \textit{select} directions which are known to lead to high curvature (e.g., \cite{jetley2018, fawzi_moosavi-dezfooli_frossard_soatto_2018}), and minimize the curvature along such chosen directions. The latter approach is more efficient, as the computation of each matrix-vector product $Hz$ involves one backward pass; focusing on high-curvature directions is therefore essential to minimize the overall curvature without having to go through each single direction in the input space. This selective approach is all the more adapted to the very sparse nature of curvature profiles we see in practice (see Fig. \ref{fig:curvature_profiles}), where only a few eigenvalues are large. This provides further motivation for identifying large curvature directions and penalizing the curvature along such directions.

Prior works in \cite{fawzi_moosavi-dezfooli_frossard_soatto_2018, jetley2018} have identified gradient directions as high curvature directions. In addition, empirical evidence reported in Section \ref{sec:theory} (Remark 3) shows a large inner product between the eigenvector corresponding to maximum eigenvalue and the gradient direction; this provides further indication that the gradient is pointing in high curvature directions, and is therefore a suitable candidate for $z$.
We set in practice $z = \frac{\text{sign} (\nabla \ell(x))}{\| \text{sign} (\nabla \ell(x)) \|}$, and finally consider the regualizer \footnote{The choice of $z \propto \nabla \ell(x)$ leads to almost identical results. We have chosen to set $z \propto \text{sign} (\nabla \ell(x))$, as we are testing the robustness of the classifier to $\ell_{\infty}$ perturbations. Hence, setting $z$ be the sign of the gradient is more relevant, as it constrains the $z$ direction to belong to the hypercube of interest.}
\begin{align*}
L_r = \| \nabla \ell(x+hz) - \nabla \ell(x ) \|^2, % \\ \text{ where } z = \text{sign} (\nabla \ell(x)).
\end{align*}
where the $\frac{1}{h^2}$ is absorbed by the regularization parameter. Our fine-tuning procedure then corresponds to minimizing the regularized loss function $\ell + \gamma L_r$ with respect to the weight parameters, where $\gamma$ controls the weight of the regularization relative to the loss term.

We stress that the proposed regularization approach significantly departs from adversarial training. In particular, while adversarial training consists in minimizing \textit{the loss on perturbed points} (which involves solving an optimization problem), our approach here consists in imposing regularity \textit{of the gradients} on a sufficiently small scale (i.e., determined by $h$). Previous works \cite{madry2017towards} have shown that adversarial training using a weak attack (such as FGSM \cite{goodfellow2014}, which involves a single gradient step) does \textit{not} improve the robustness. We show that our approach, which rather imposes gradient regularity (i.e., small curvature) along such directions, does lead to a significant improvement in the robustness of the network.

We use two pre-trained networks, ResNet-18~\cite{he2015deep} and WResNet-28x10~\cite{zagoruyko2016wide}, on the CIFAR-10 and SVHN datasets, where the pixel values are in $[0, 255]$.
For the optimization of the regularized objective, we use the Adam optimizer with a decreasing learning rate between [$10^{-4}$, $10^{-6}$] for a duration of $20$ epochs starting from a pre-trained network. We linearly increase the value of $h$ from $0$ to $1.5$ during the first $5$ epochs, and from there on, we use a fixed value of $h=1.5$. For $\gamma$, we set it to $4$ and $8$ for ResNet-18 and WResNet-28 respectively.
\begin{figure}
    \centering
    \includegraphics[width=0.7\columnwidth]{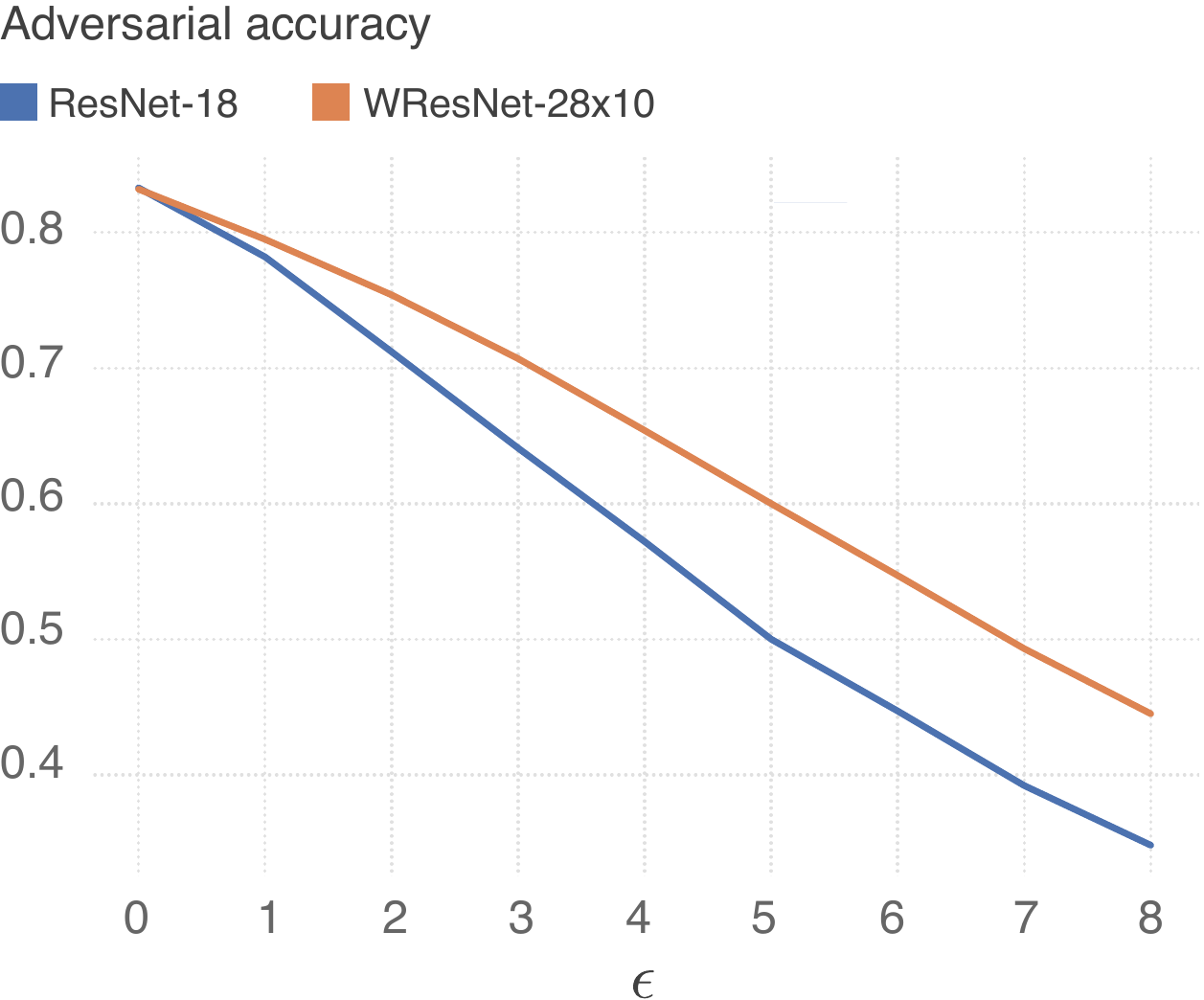}
    \caption{Adversarial accuracy versus perturbation magnitude $\epsilon$ computed using PGD(20), for ResNet-18 and WResNet-28x10 trained with CURE on CIFAR-10. See \cite{madry2017towards} for the curve corresponding to adversarial training. Curve generated for 2000 random test points.}
    \label{fig:cifar_robustness_vs_eps}
\end{figure}

\begin{figure}
    \centering
    \includegraphics[width=0.9\columnwidth]{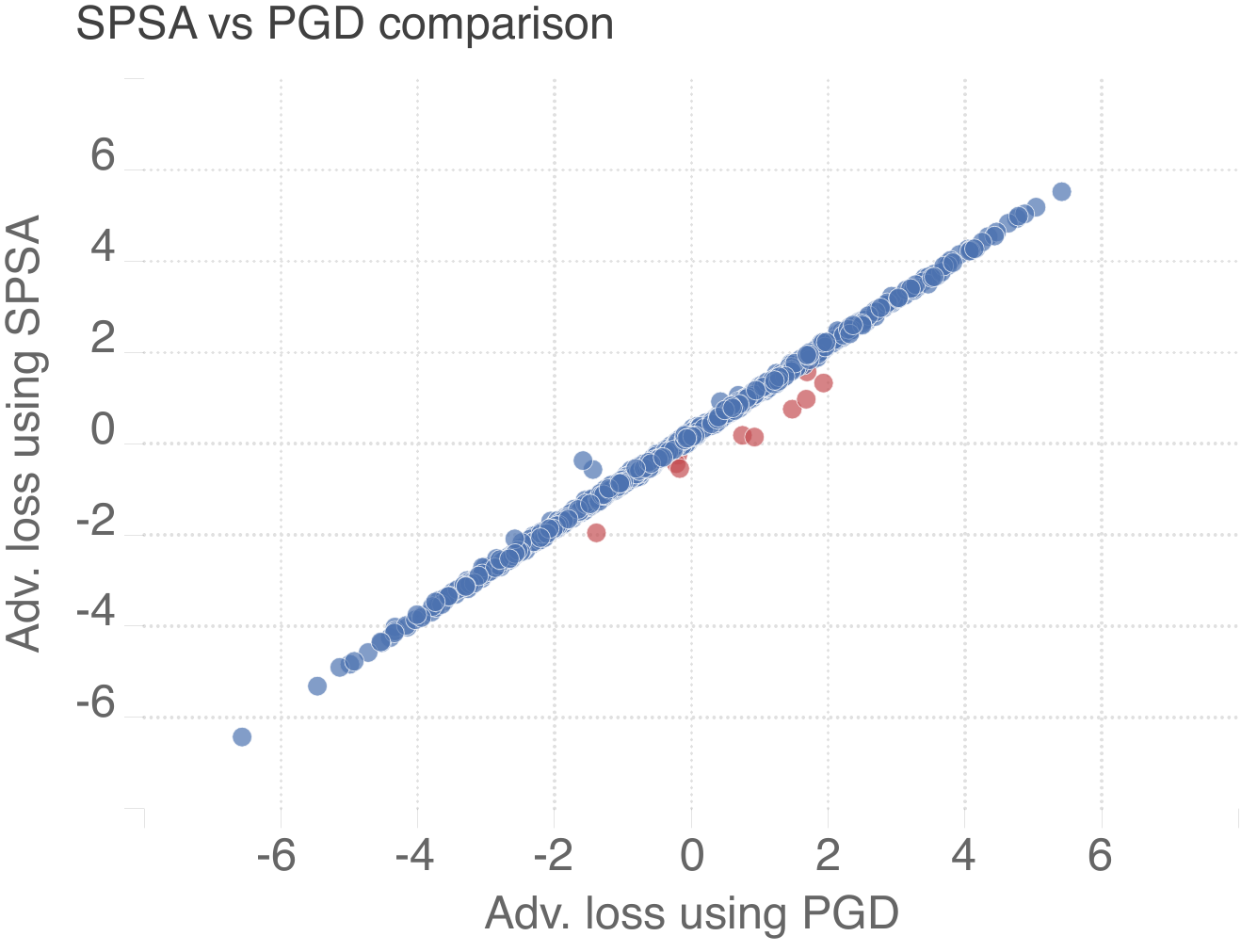}
    \caption{Analysis of gradient masking in a network trained with CURE. Adversarial loss computed with SPSA (y-axis) vs. adversarial loss with PGD(100) (x-axis) on a batch of 1000 datapoints. Adversarial loss corresponds to the difference of logits on true and adversarial class. Each point in the scatter plot corresponds to a single test sample. Negative loss indicates that the data point is misclassified. Points close to the line $y = x$ indicate that both attacks identified similar adversarial perturbations. Points below the line, shown in red, indicate points for which SPSA identified stronger adversarial perturbation than PGD. Note that overall, SPSA and PGD identified similarly perturbations.}
    \label{fig:spsa_pgd}
\end{figure}
\begin{figure}
    \centering
    \begin{subfigure}[t]{0.20\textwidth}
        \includegraphics[width=\linewidth]{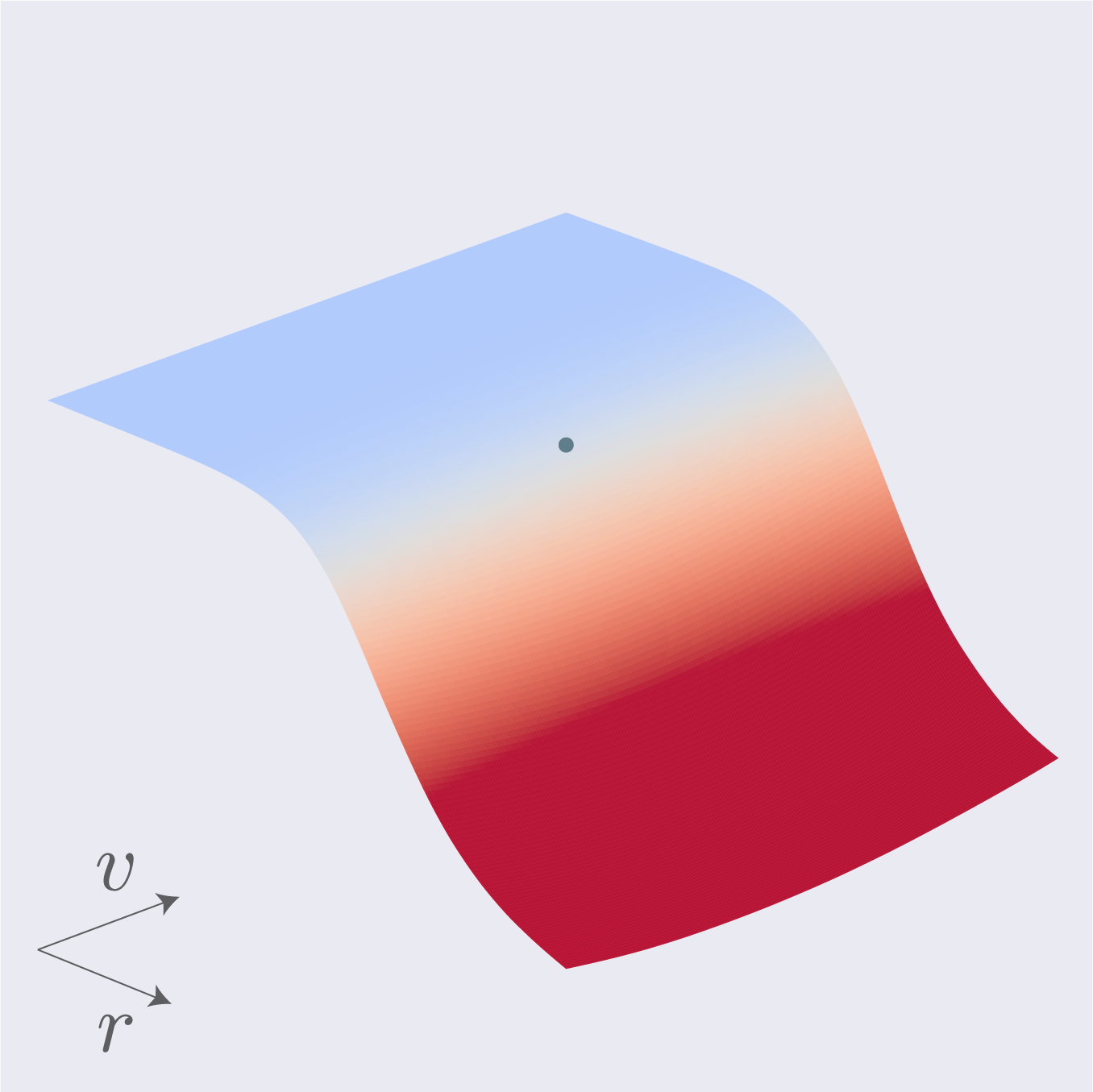}
        \caption{ResNet-18}
    \end{subfigure}~
    \begin{subfigure}[t]{0.20\textwidth}
        \includegraphics[width=\linewidth]{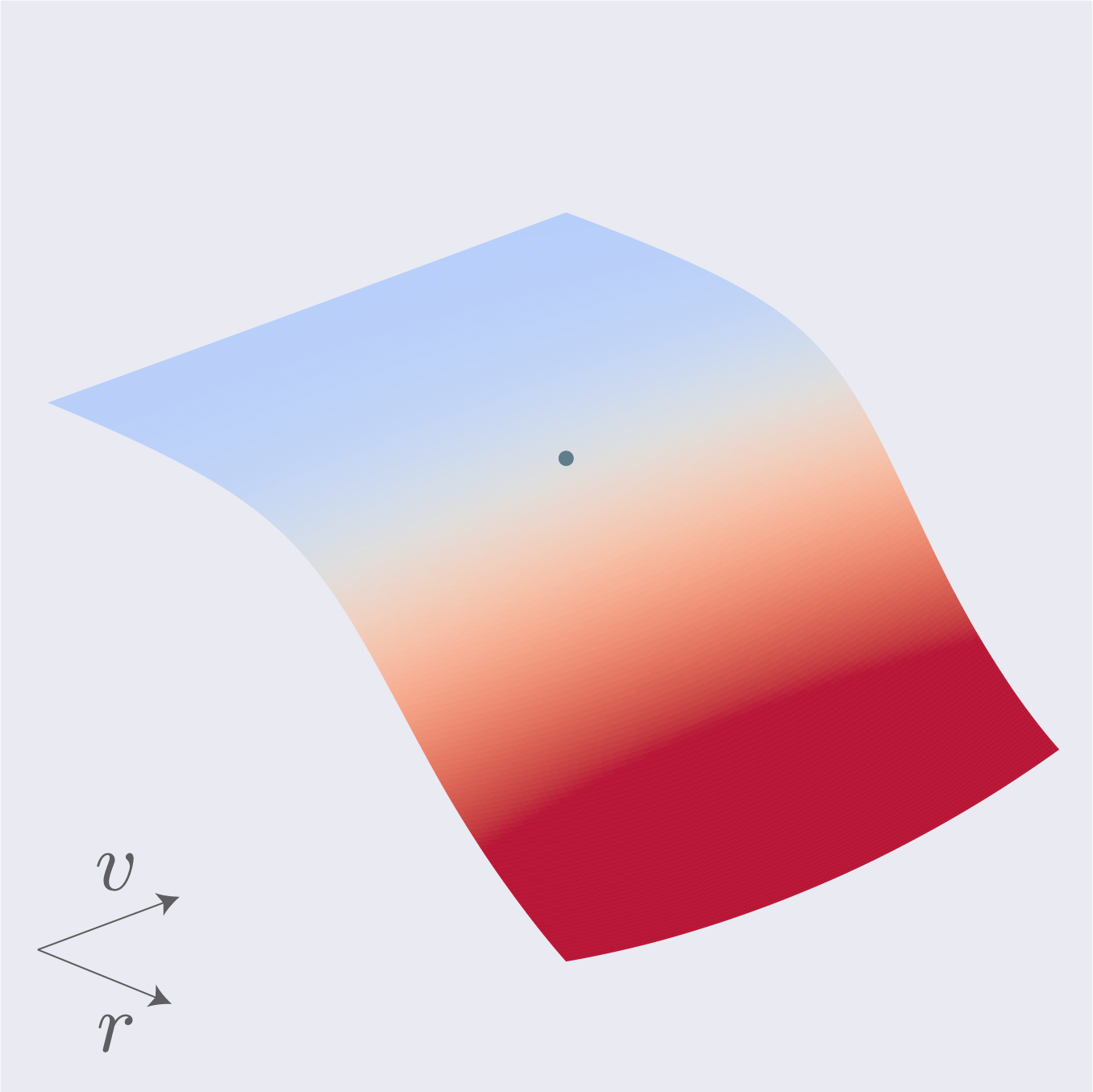}
        \caption{WideResNet-28}
    \end{subfigure}
    \caption{Similar plot to Fig. \ref{fig:loss_surfaces}, but where the loss surfaces of the network obtained with CURE are shown.}
    \label{fig:loss_surfaces_flattening}
\end{figure}
\begin{figure}
    \centering
    \includegraphics[width=0.8\linewidth]{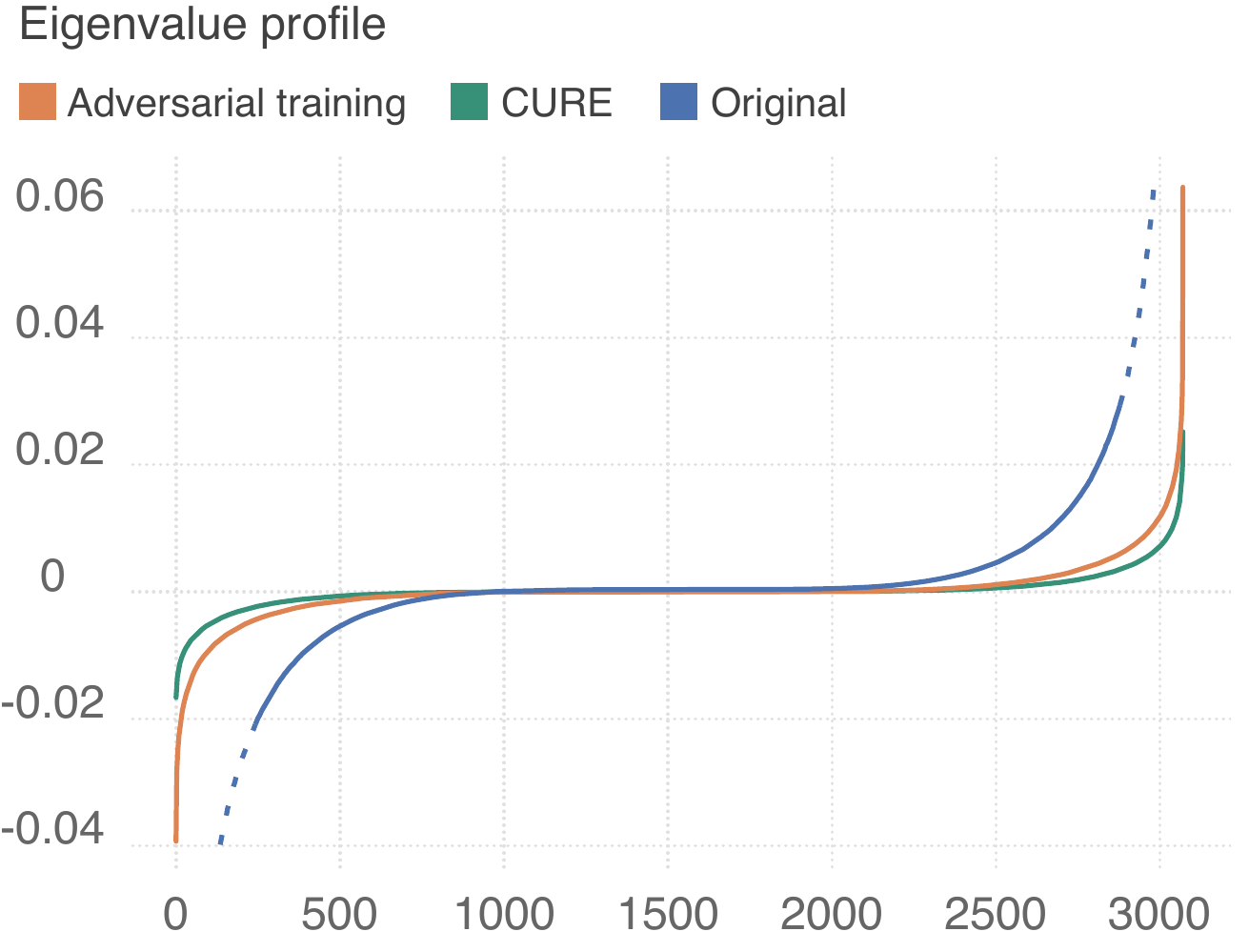}
    \caption{Curvature profile for a network fine-tuned using adversarial training and CURE. The ResNet-18 architecture on CIFAR-10 is used. For comparison, we also report the profile for the original network (same as Fig. \ref{fig:curvature_profiles}), where we clipped the values to fit in the $y$ range.}
    \label{fig:curvature_cure}
\end{figure}
\paragraph{Results.}
We evaluate the regularized networks with a strong PGD attack of 20 iterations, as it has been shown to outperform other adversarial attack algorithms~\cite{madry2017towards}. The adversarial accuracies of the regularized networks are reported in Table~\ref{tab:cifar_regularization_result} for CIFAR-10, and in the supp. material for SVHN. Moreover, the adversarial accuracy as a function of the perturbation magnitude $\epsilon$ is reported in Fig. \ref{fig:cifar_robustness_vs_eps}.

Observe that, while networks trained on the original dataset are not robust to perturbations as expected, performing 20 epochs of fine-tuning with the proposed regularizer leads to a significant boost in adversarial performance. In particular, the performance with the proposed regularizer is comparable to that of adversarial training reported in \cite{madry2017towards}. This result hence shows the importance of the curvature decrease phenomenon described in this paper in explaining the success of adversarial training.

In addition to verifying our claim that small curvature confers robustness to the network (and that it is the underlying effect in adversarial training), we note that the proposed regularizer has practical value, as it is efficient to compute and can therefore be used as an alternative to adversarial training. In fact, the proposed regularizer requires $2$ backward passes to compute, and is used in fine-tuning for $20$ epochs. In contrast, one needs to run adversarial training against a \textit{strong} adversary in order to reach good robustness \cite{madry2017towards}, and start the adversarial training procedure from scratch. We note that strong adversaries generally require around $10$ backward passes, making the proposed regularization scheme a more efficient alternative. We note however that the obtained results are slightly worse than adversarial training; we hypothesize that this might be either due to higher order effects in adversarial training not captured with our second order analysis or potentially due to a sub-optimal choice of hyper-parameters $\gamma$ and $h$.

\begin{figure*}
    \centering
    \includegraphics[width=0.7\linewidth]{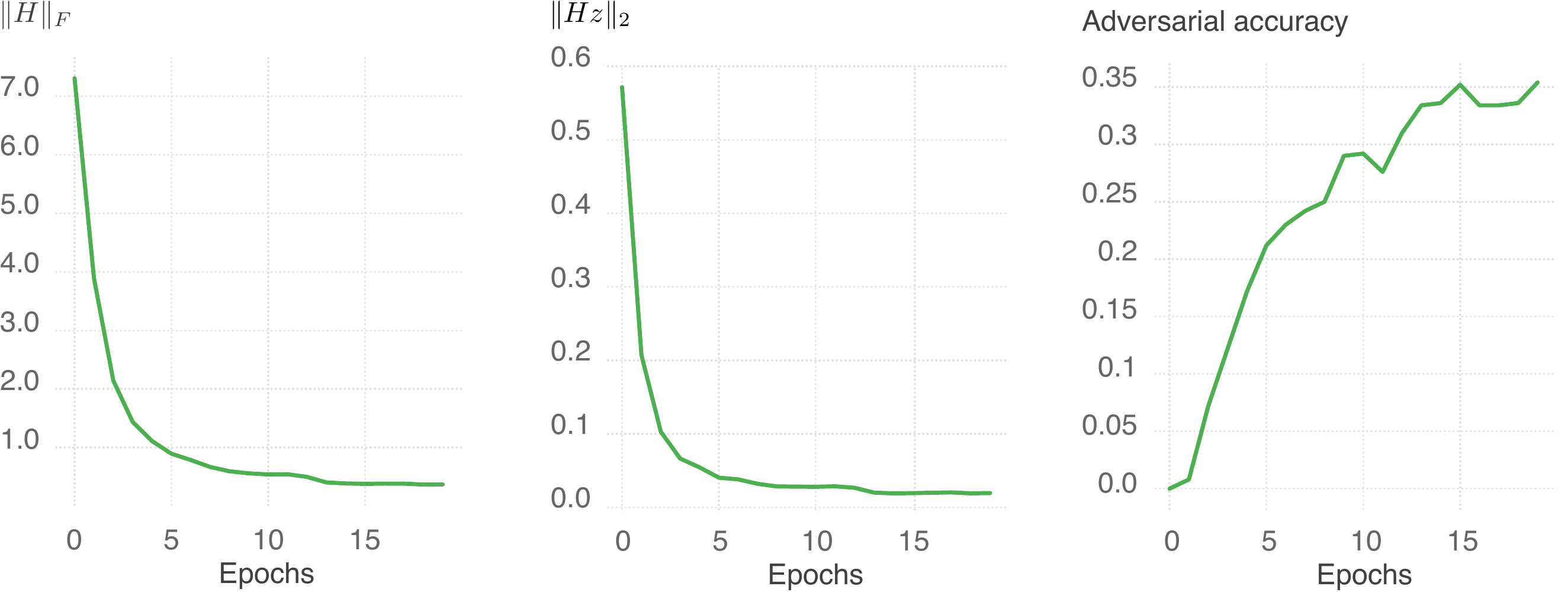}
    \caption{Evolution throughout the course of CURE fine-tuning for a ResNet-18 on CIFAR-10. The curves are averaged over $1000$ datapoints. \textbf{Left: } estimate of Frobenius norm , \textbf{Middle: } $\| H z \|$, where $z = \text{sign} (\nabla \ell(x))/\|\text{sign} (\nabla \ell(x))\|_2$ and \textbf{Right:} adversarial accuracy computed using PGD(20). The Frobenius norm is estimated with $\| H \|_F^2 = \mathbb{E}_{z \sim \mathcal{N}(0,I)} \| H z \|^2$, where the expectation is approximated with an empirical expectation over $100$ samples $z_i \sim \mathcal{N}(0, I)$.}
    \label{fig:evolution}
\end{figure*}
\paragraph{Stronger attacks and verifying the absence of gradient masking.} To provide further evidence on the robustness of the network fine-tuned with CURE, we attempt to find perturbations for the network with more complex attack algorithms. For the WideResNet-28x10, we obtain an addversarial accuracy of $41.1\%$ on the test set when using PGD(40) and PGD(100). % the adversarial accuracy on the test set with PGD(40) and PGD(100) are equal to 41.1\%,
This is only slightly worse than the result reported in  Table \ref{tab:cifar_regularization_result} with PGD(20). This shows that increasing the complexity of the attack does not lead to a significant decrease in the adversarial accuracy. Moreover, we evaluate the model against a gradient-free optimization method (SPSA), similar to the methodology used in  \cite{uesato2018adversarial}, and obtained an adversarial accuracy of 44.5\%. We compare moreover in Fig.~\ref{fig:spsa_pgd} the \textit{adversarial loss} (which represents the difference between the logit scores of the true and adversarial class) computed using SPSA and PGD for a batch of test data points. Observe that both methods lead to comparable adversarial loss (except on a few data points), hence further justifying that CURE truly improves the robustness, as opposed to masking or obfuscating gradients.
Hence, just like adversarial training which was shown empirically to lead to networks that are robust to all tested attacks in \cite{uesato2018adversarial, athalye2018obfuscated}, our experiments show that the regularized network has similar robustness properties.

\paragraph{Curvature and robustness.}
We now analyze the network obtained using CURE fine-tuning, and show that the obtained network has similar geometric properties to the adversarially trained one. Fig.~\ref{fig:loss_surfaces_flattening} shows the loss surface in a plane spanned by $(r, v)$, where $r$ and $v$ denote respectively a normal to the decision boundary and a random direction. Note that the loss surface obtained with CURE is qualitatively very similar to the one obtained with adversarial training (Fig.~ \ref{fig:loss_surfaces}), whereby the loss has a more linear behavior in the vicinity of the data point. Quantitatively, Fig.~\ref{fig:curvature_cure} compares the curvature profiles for the networks trained with CURE and adversarial fine-tuning. Observe that both profiles are very similar. We also report the evolution of the adversarial accuracy and curvature quantities in Fig.~\ref{fig:evolution} during fine-tuning with CURE. Note that throughout the fine-tuning process, the curvature decreases while the adversarial accuracy increases, which further shows the link between robustness and curvature.
Note also that, while we explicitly regularized for $\| H z \|$ (where $z$ is a fixed direction for each data point) as a proxy for $\| H \|_F$, the network does show that the intended target $\| H \|_F$ decreases in the course of training,
hence further suggesting that $\| H z \|$ acts as an efficient proxy of the global curvature.

\begin{figure}[t]
    \centering
    \includegraphics[width=\linewidth]{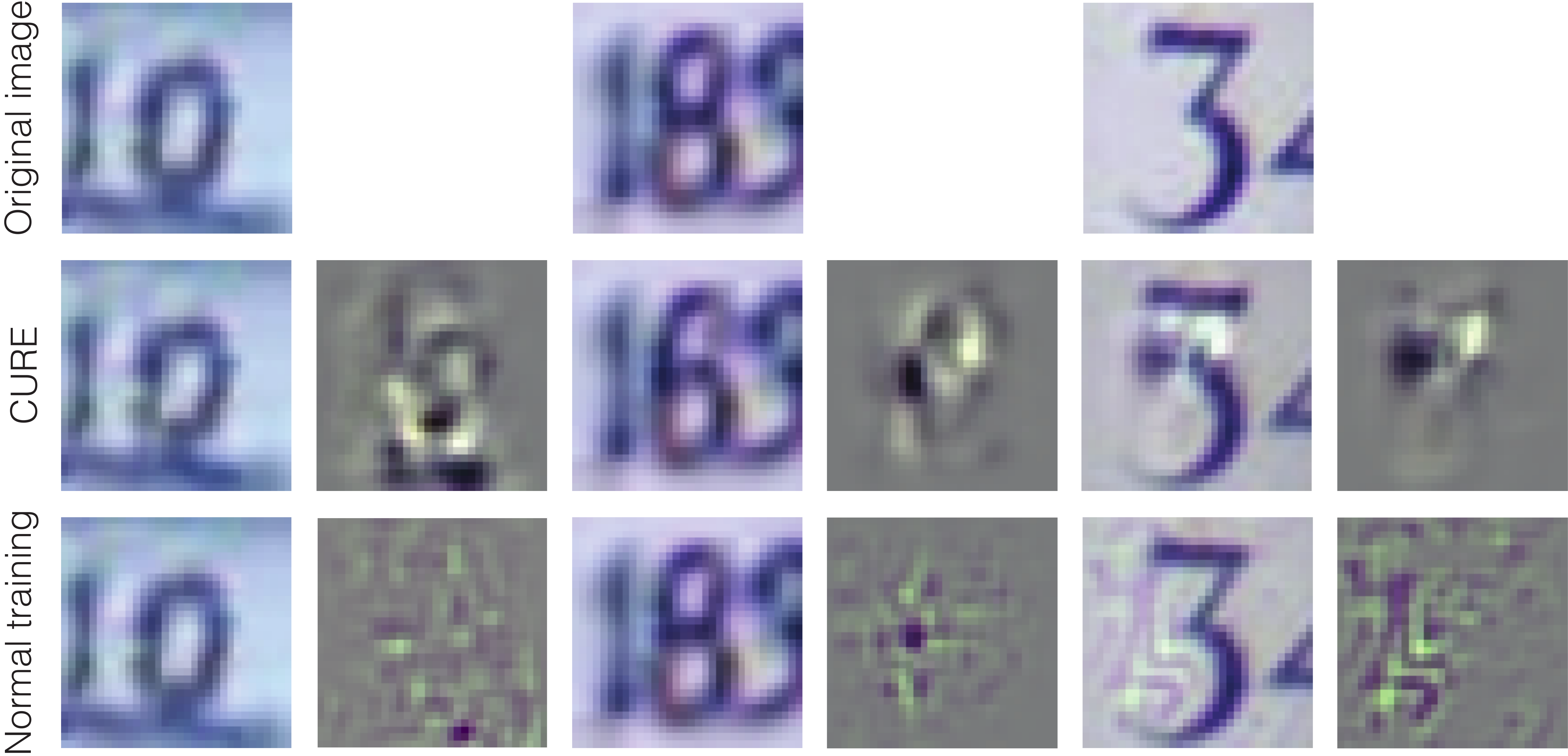}
    \caption{Visualizations of perturbed images and perturbations on SVHN for the ResNet-18 classifier.}
    \label{fig:illust_adv_perts}
\end{figure}
\paragraph{Qualitative evaluation of adversarial perturbations.} We finally illustrate some adversarial examples in Fig.~\ref{fig:illust_adv_perts} for networks trained on SVHN. Observe that the network trained with CURE exhibits visually meaningful adversarial examples, as perturbed images do resemble images from the adversary class. A similar observation for adversarially trained models has been made in \cite{tsipras2018robustness}.

\section{Conclusion}

Guided by the analysis of the geometry of adversarial training, we have provided empirical and theoretical evidence showing the existence of a strong correlation between small curvature and robustness. To validate our analysis, we proposed a new regularizer (CURE), which directly encourages small curvatures (in other words, promotes local linearity). This regularizer is shown to significantly improve the robustness of deep networks and even achieve performance that is comparable to adversarial training. In light of prior works attributing the vulnerability of classifiers to the ``linearity of deep networks'', this result is somewhat surprising, as it shows that one needs to decrease the curvature (and not increase it) to improve the robustness. In addition to validating the importance of controlling the curvature for improving the robustness, the proposed regularizer also provides an efficient alternative to adversarial training. In future work, we plan to leverage the proposed regularizer to train provably robust networks.
\section*{Acknowledgements}
A.F. would like to thank Neil Rabinowitz and Avraham Ruderman for the fruitful discussions. S.M and P.F would like to thank the Google Faculty Research Award, and the Hasler Foundation, Switzerland, in the framework of the ROBERT project.

{\small
\bibliographystyle{ieee}
\bibliography{egbib}
}
\onecolumn
\appendix
\section{Supplementary material}

\subsection{Results of applying CURE on the SVHN dataset}
We fine-tune a pre-trained ResNet-18 using our method, CURE, on SVHN dataset. The learning rate is varying between $[10^{-4},10^{-6}]$ for a duration of 20 epochs. The value of $\gamma$ is set to $4$, $8$, and $12$ for 10, 5, and 5 epochs respectively. Also, for SVHN, we fix $h=1.25$.

\begin{table}[ht]
    \centering
    \begin{tabular}{lcc}
        \toprule
        & \multicolumn{2}{c}{ResNet-18}\\
        \cmidrule{2-3}
         & Clean & Adversarial\\
        \midrule
        Normal training & $96.3\%$ & $0.9\%$\\
        CURE & $91.1\%$ & $28.4\%$\\
        Adv. training (reported in ~\cite{buckman2018thermometer}) & $93\%$ & $33\%$\\
        \bottomrule
    \end{tabular}
    \caption{Adversarial and clean accuracy for SVHN for original, regularized and adversarially trained models. Performance is reported for a ResNet-18 model, and the perturbations are computed using PGD(10) with $\epsilon=12$.}
    \label{tab:svhn_regularization_result}
\end{table}

\subsection{Curvature profile of CURE}
\begin{figure}[h]
    \centering
    \includegraphics[width=0.5\linewidth]{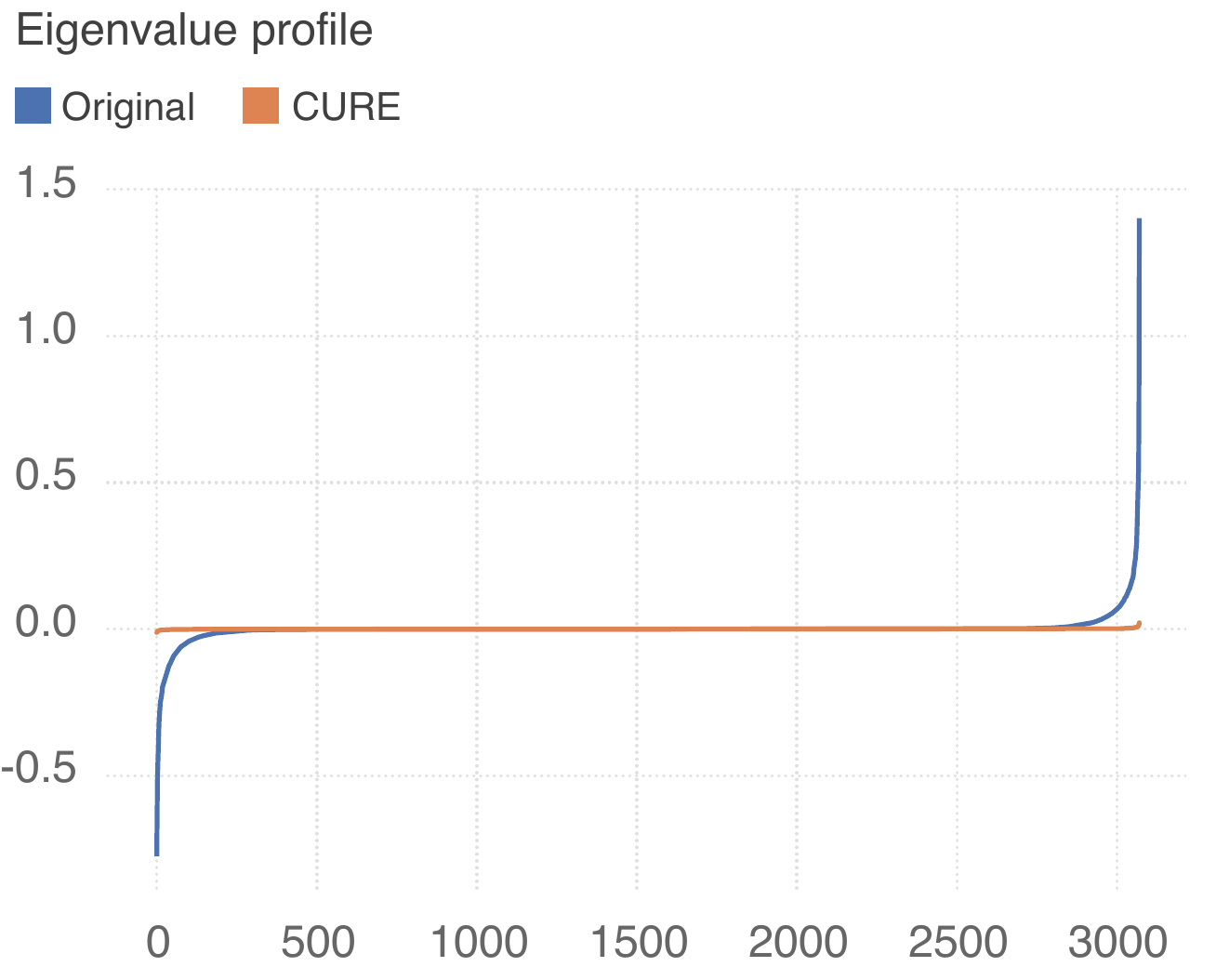}
    \caption{Curvature profiles for a ResNet-18 model trained on SVHN and its fine-tuned version using CURE.}
    \label{fig:curvature_cure_svhn}
\end{figure}
\pagebreak
\subsection{Loss surface visualization}

\begin{figure}[h]
    \centering
    \begin{subfigure}[]{0.20\textwidth}
        \includegraphics[width=\linewidth]{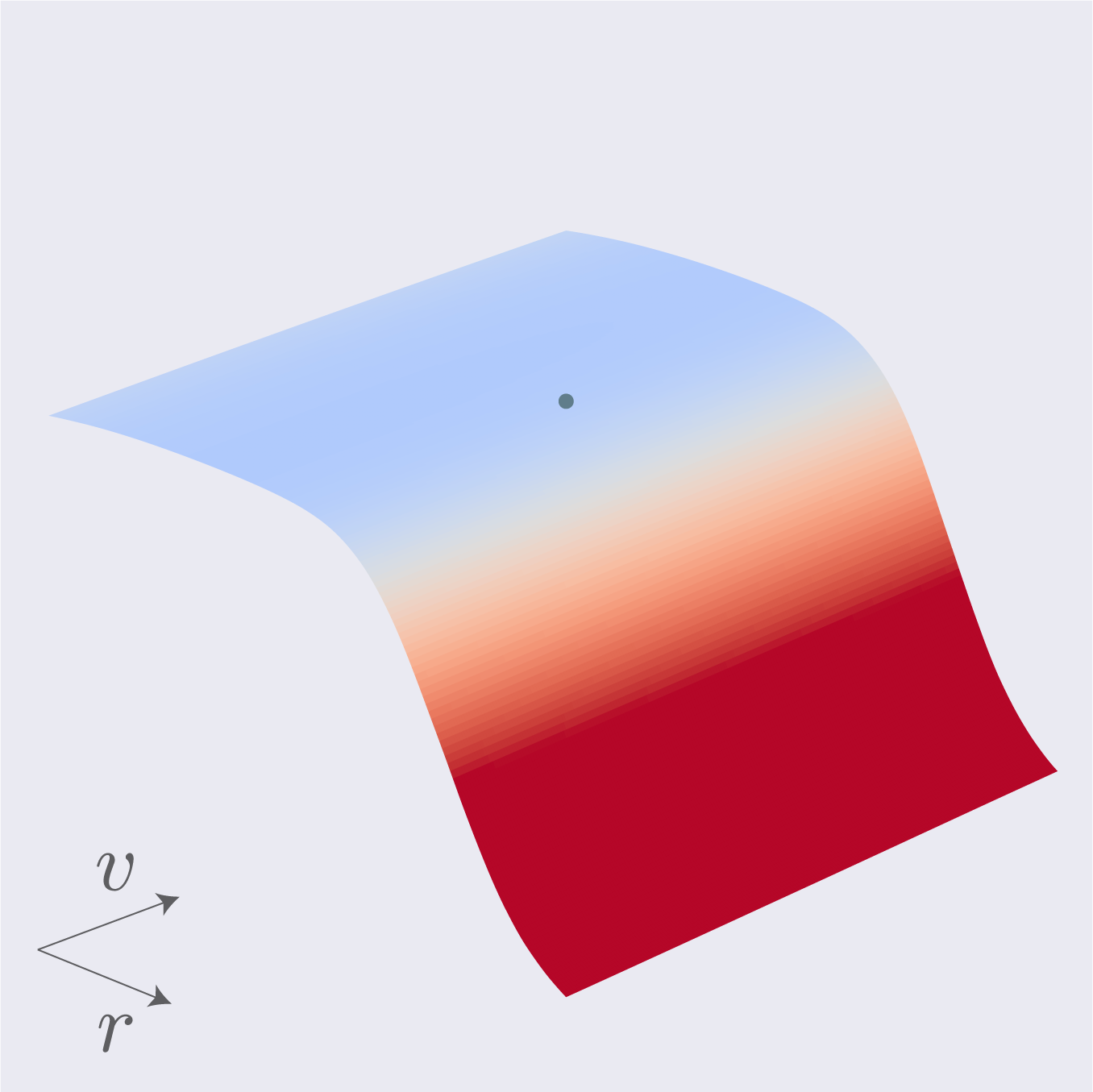}
        \caption{CURE}
    \end{subfigure}~
    \begin{subfigure}[]{0.20\textwidth}
        \includegraphics[width=\linewidth]{svhn_normal_loss_surface.png}
        \caption{Original network}
    \end{subfigure}
    \caption{Illustration of the negative of the loss surface of the original and the fine-tuned networks trained on SVHN; i.e., $-\ell(s)$ for points $s$ belonging to a plane spanned by a normal direction $r$ to the decision boundary, and random direction $v$. The original sample is illustrated with a blue dot. The light blue part of the surface corresponds to low loss (i.e., corresponding to the classification region of the sample), and the red part corresponds to the high loss (i.e., adversarial region).}
    \label{fig:loss_surfaces_flattening_svhn}
\end{figure}

\subsection{Evolution of curvature and robustness}
\begin{figure}[h]
    \centering
    \includegraphics[width=0.8\linewidth]{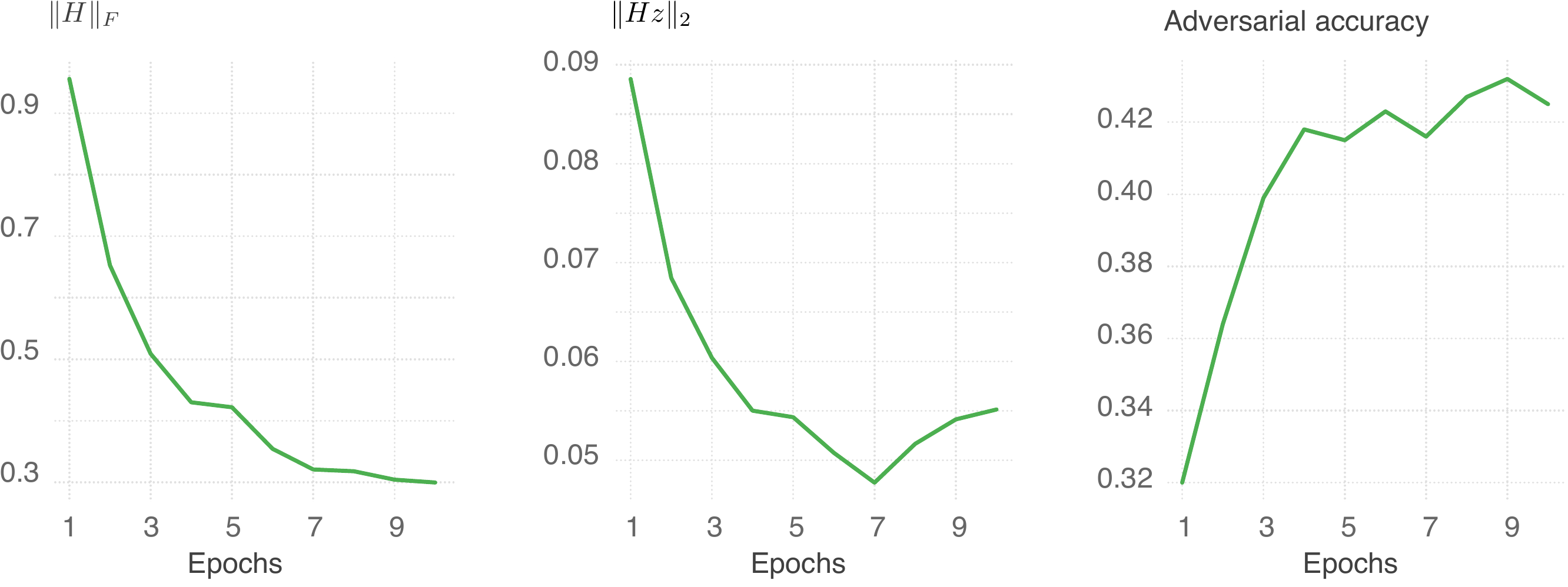}
    \caption{Evolution throughout the course of our CURE fine-tuning for a ResNet-18 on SVHN. The curves are averaged over $1000$ datapoints. \textbf{Left: } estimate of Frobenius norm , \textbf{Middle: } $\| H z \|$, where $z=\text{sign} (\nabla \ell(x))=\|\text{sign}(\nabla \ell(x))\|_2$, and \textbf{Right:} adversarial accuracy computed using PGD(10) with $\epsilon=8$. The Frobenius norm is estimated with $\| H \|_F^2 = \mathbb{E}_{z \sim \mathcal{N}(0,I)} \| H z \|^2$, where the expectation is approximated with an empirical expectation over $100$ samples $z_i \sim \mathcal{N}(0, I)$.}
\end{figure}

\pagebreak
\subsection{Adversarial accuracy}
\begin{figure}[h]
    \centering
    \includegraphics[width=0.5\columnwidth]{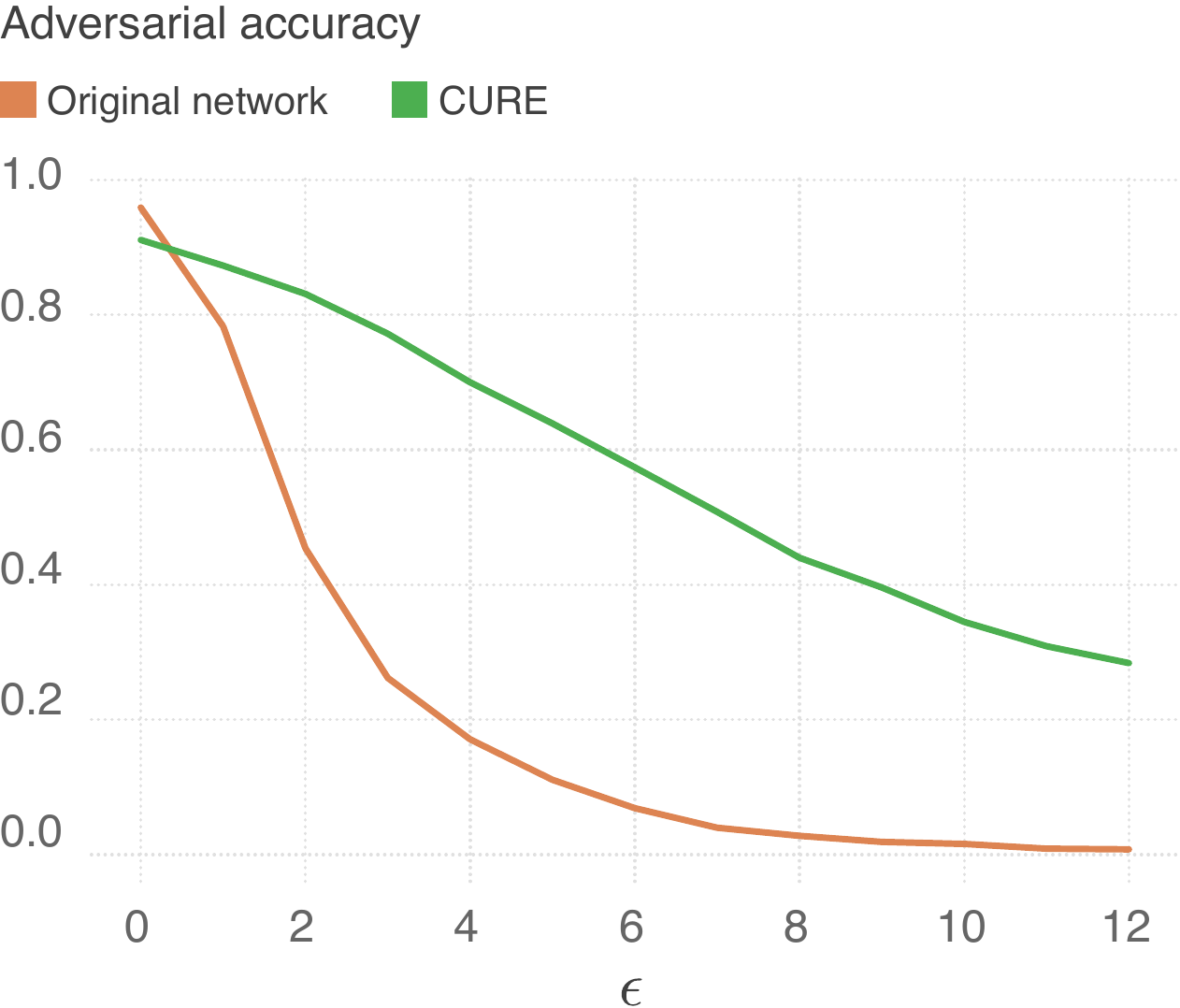}
    \caption{Adversarial accuracy versus perturbation magnitude $\epsilon$ computed using PGD(10), for ResNet-18 trained with CURE on SVHN. Curve generated for 2000 random test points.}
    \label{fig:svhn_robustness_vs_eps}
\end{figure}
\end{document}